\relax
\documentclass[letterpaper]{article} 
\usepackage{times}  
\usepackage{helvet}  
\usepackage{courier}  
\usepackage{url}  
\usepackage{graphicx}  

\usepackage{natbib}
\usepackage{times}  
\usepackage{helvet}  
\usepackage{courier}  
\usepackage{url}  
\usepackage{subcaption}
\usepackage{lmodern}
\usepackage{datatool}
\usepackage[margin=1in]{geometry}
\usepackage{algorithm}
\usepackage{amsthm,amsmath,amssymb}
\usepackage{bm}
\usepackage{latexsym}
\usepackage{float}
\usepackage{color}
\usepackage{pifont}
\usepackage{amsthm}
\usepackage{amsmath,amssymb} 
\usepackage{bm} 

\newtheorem{theorem}{Theorem}

\newtheorem{property}{Prop.}
\newtheorem{lemma}{Lemma}
\newtheorem{definition}{Def.}
\newtheorem{assumption}{Assumption}

\usepackage{hyperref}

\author{Takeshi Teshima,\textsuperscript{\rm{1,2}}
Miao Xu,\textsuperscript{\rm{2}}
Issei Sato,\textsuperscript{\rm{1,2}}
Masashi Sugiyama\textsuperscript{\rm{2,1}}\\
\textsuperscript{\rm{1}}The University of Tokyo\\
\textsuperscript{\rm{2}}RIKEN \\
teshima@ms.k.u-tokyo.ac.jp,
miao.xu@riken.jp,
\{sato, sugi\}@k.u-tokyo.ac.jp}

\newcommand{\hide}[1]{}
\usepackage{hyperref}

\usepackage{algorithm}
\usepackage{algorithmic}
\usepackage{tikz}
\usetikzlibrary{positioning}

\usepackage{xcolor}
\newcommand{\quotes}[1]{``#1''}
\usepackage{subcaption}
\def \pathIllust {.}
\def \pathSynth {.}
\def \pathVisual {.}
\usepackage{dsfont}
\usepackage{bbm}
\def \Reg {\mathcal{R}}

\def \O {\mathbf{O}}
\def \mI {\mathbf{I}}

\def \X {\mathbf{X}}
\def \Indicator {\mathbbm{1}}
\def \I {\mathbf{I}}

\def \Z {\mathbf{Z}}
\def \W {\mathbf{W}}
\def \H {\mathbf{H}}
\def \Y {\mathbf{Y}}
\def \E {\mathbb{E}}
\def \P {\mathbb{P}}
\def \Proj {\mathcal{P}}

\def \Prob {\mathbb{P}}
\def \mP {\mathbf{P}}
\def \Q {\mathbf{Q}}
\def \M {\mathbf{M}}
\def \hM {\widehat{M}}

\def \L {\mathcal{L}}
\def \LO {\mathcal{L}_\Omega}

\def \A {\mathbf{A}}
\def \B {\mathbf{B}}
\def \U {\mathbf{U}}
\def \V {\mathbf{V}}
\def \mE {\mathbf{E}}
\def \mB {\mathbf{B}}
\def \mSigma{\mathbf{\Sigma}}
\def \bSigma{\bm\Sigma}
\def \p{\bm{p}}
\def \q{\bm{q}}
\def \e{\bm{e}}
\def \f{\bm{f}}
\def \clipped{\mathrm{c}}
\def \test{\mathrm{t}}
\def \valid{\mathrm{v}}
\def \F {\mathbf{F}}
\def \G {\mathbf{G}}

\def \bv {\bm{v}}
\def \Cf {\mbox{Clip}}
\def \svd{\mathrm{svd}}
\def \F {\mathrm{F}}
\def \Fro {\mathrm{F}}
\def \trnrm {\mathrm{tr}}
\def \op {\mathrm{op}}
\def \trace {\mathrm{tr}}
\def \hadamard {\odot}

\def \I {\mathcal{I}}
\def \P {\mathcal{P}}
\def \PA {\mathcal{P}_\A}
\def \PB {\mathcal{P}_\B}
\def \PC {\mathcal{P}^\mathrm{*}}
\def \PCij {P^\mathrm{*}_{ij}}

\def \PO {\mathcal{P}_\Omega}
\def \POk {\mathcal{P}_{\Omega_k}}
\def \POminusC {\mathcal{P}_{\Omega\setminus\mathcal{C}}}
\def \POC {{\mathcal{P}_{\Omega}^\mathrm{*}}}

\def \RB {\mathcal{R}_\C}
\def \RO {\mathcal{R}_\Omega}
\def \ROh {\RO^\half}
\def \ROk {\mathcal{R}_{\Omega_k}}
\def \ROkC {\mathcal{R}_{\Omega_k}^{\mathrm{*}}}
\def \PT {\mathcal{P}_T}
\def \PU {\mathcal{P}_\mathrm{U}}
\def \PV {\mathcal{P}_\mathrm{V}}
\def \PTT {\mathcal{P}_{T^\perp}}
\def \PnC {\mathcal{P}_{\TrueNonClipped}}
\def \A {{\Omega \setminus \mathcal{C}}}
\def \B {\mathcal{C}}
\def \C {\mathcal{C}}
\def \mE {\mathbf{E}}

\def \X {\mathbf{X}}
\def \Y {\mathbf{Y}}
\def \Z {\mathbf{Z}}
\def \W {\mathbf{W}}
\def \Wk {\mathbf{W}_k}
\def \Wkminus {\mathbf{W}_{k-1}}
\def \H {\mathbf{H}}
\def \M {\mathbf{M}}
\def \Mc {\mathbf{M}^\c}
\def \Mcij {{M^\c_{ij}}}
\def \Mij {{M_{ij}}}
\def \hM {\widehat{\M}}
\def \mO {\mathbf{O}}
\def \U {\mathbf{U}}
\def \V {\mathbf{V}}
\def \S {\mathbf{S}}
\def \UV {\U\V^\top}

\def \e {\bm{e}}
\def \f {\bm{f}}
\def \Eij {\e_i \f_j^\top}
\def \Eab {\e_a \f_b^\top}
\def \Dk {\Delta_k}
\def \Dkminus {\Delta_{k-1}}
\def \half {\frac{1}{2}}
\def \oij{\omega_{ij}}
\def \oijk{\omega_{ij}^{(k)}}
\def \bcij {{\bar{c}_{ij}}}
\def \bomega {\bm{\omega}}
\def \c{\mathrm{c}}
\def \Cf{\mathrm{Clip}}
\def \matsp{\mathbb{R}^{n_1 \times n_2}}
\def \indsp{[n_1] \times [n_2]}
\def \range {\mathrm{range}}

\def \sumij {\sum_{(i, j)}}
\def \sumk {\sum_{k=1}^L}
\def \fullij {{(i, j) \in [n_1] \times [n_2]}}
\def \sumO {\sum_{(i, j) \in \Omega}}
\def \maxij {\max_{(i, j)}}
\def \supG {{\sup_{\X \in G}}}
\def \Order {\mathcal{O}}

\def \normbar {\|}
\def \pmin {p_{\mathrm{min}}}
\def \nuC {\nu_{\TrueNonClipped}}
\def \rhoFro {{\rho_\Fro}}
\def \rhoInfty {{\rho_\infty}}
\def \rhoOp {{\rho_\op}}

\def \TrueNonClipped {{\mathcal{B}}}
\def \McOmega {{\Mc_\Omega}}
\def \fone {{\mathrm{f}_1}}
\def \Uij {{\overline{C}_{ij}}}
\def \Lij {{\underline{C}_{ij}}}
\def \UCoherence {{\mu^\mathrm{U}}}
\def \VCoherence {{\mu^\mathrm{V}}}
\def \unnCoherence {{\mu}}
\def \jointCoherence {{\mu_1}}
\def \USp{\mathbb{R}^{n_1 \times r}}
\def \SigmaSp{\mathbb{R}^{r \times r}}
\def \VSp{\mathbb{R}^{r \times n_2}}
\def \Psp{\mathbb{R}^{n_1 \times k}}
\def \Qsp{\mathbb{R}^{n_2 \times k}}
\def \fMC {f^{\mathrm{MC}}}
\def \fCMC {f^{\mathrm{CMC}}}
\def \pMinFro {\pmin^{\mathrm{F}}}
\def \pMinOpOne {\pmin^{\mathrm{op,1}}}
\def \pMinOpTwo {\pmin^{\mathrm{op,2}}}
\def \pMinInfty {\pmin^{\infty}}
\def \pMinMain {\pmin^{\mathrm{main}}}
\def \muG {{\mu_G}}
\def \cAllRho {{c_\rho}}

\def \uvupperbound {\frac{\sqrt{p}}{2 \sqrt{2}}}
\def \opnormupperbound {\frac{1}{2}}
\def \kZeroValue {\left\lceil \log_2 (2\sqrt{2}\sqrt{n_1 n_2 r})\right\rceil}
\def \halfPowerKZeroValue {{\sqrt{\frac{1}{n_1 n_2 r}} \frac{1}{2 \sqrt{2}}}}
\def \uvupperboundTimesSqrtR {\sqrt{\frac{p}{r}} \frac{1}{2 \sqrt{2}}}
\def \constFeasibilityOne {\max\biggl\{\frac{24}{(1/2 - \rhoFro)^2}, \frac{8}{(1/4 - \rhoOp)^2},}
\def \constFeasibilityTwo {\frac{8}{(1/2 - \rhoInfty)^2}, \frac{8}{(1/2 - \nuC)^2}\biggr\}}
\usepackage[toc,page]{appendix}
\usepackage{float}
\newfloat{algorithm}{t}{lop}
\def \mEP {{\mathbf{E}_\mathrm{P}}}
\def \mEQ {{\mathbf{E}_\mathrm{Q}}}
\def \kzeroSubstitution {\left(\alpha + \frac{\log(n_1 n_2)}{2 \log 2}\right)}
\def \alphaValue {\log_2 (2 \sqrt{2r}) + 1}
\def \cSubstitution {\cAllRho}
\def \orderSubstitution {\Order\left(\frac{(n_1 + n_2) \log(n_1 n_2)^2}{n_1 n_2}\right)}
\def \lemMainDelta {2 (n_1n_2)^{1 -\beta}}
\def \lemMainPLowerBound {\frac{8 \beta r \mu_0}{3(1/2 - \nuC)^2} \frac{(n_1 + n_2) \log(n_1n_2)}{n_1 n_2}}
\def \lemMainBetaMin {{1 + (\log 2 / \log(n_1 n_2))}}
\def \lemFroDelta {e^{\frac{1}{4}}(n_1 n_2)^{- \beta}}
\def \lemFroPLowerBound {\frac{8 k_0 \mu_0 \beta r}{(1/2 - \rhoFro)^2} \frac{(n_1 + n_2) \log(n_1 n_2)}{n_1 n_2}}
\def \lemFroQLowerBound {{\frac{8 \mu_0 r}{(1/2 - \rhoFro)^2} \beta \log(n_1 n_2) \frac{n_1 + n_2}{n_1 n_2}}}
\def \lemFroBetaMin {{1 / (4 \log(n_1 n_2))}}
\def \lemOpDelta {(n_1 + n_2)^{1 - \beta}}
\def \lemOpPLowerBoundForBernsteinCondition {{\frac{8 k_0 \beta}{3 (1/4 - \rhoOp)^2} \frac{\log(n_1 + n_2)}{\max(n_1, n_2)}}}
\def \lemOpPLowerBoundForBound {{\frac{8 k_0 \beta r \jointCoherence^2}{3 (1/4 - \rhoOp)^2} \frac{\max(n_1, n_2) \log(n_1 + n_2)}{n_1 n_2}}}
\def \lemOpQLowerBoundForBernsteinCondition {{\frac{8 \beta}{3} \frac{\log(n_1 + n_2)}{\max(n_1, n_2)}}}
\def \lemOpQLowerBoundForBound {{\frac{8 \beta r \jointCoherence^2}{3 (1/4 - \rhoOp)^2} \frac{\max(n_1, n_2) \log(n_1 + n_2)}{n_1 n_2}}}
\def \lemOpBetaMin {{1}}
\def \lemOpBoundCoeff {{\sqrt{\frac{8}{3} \beta \frac{\max(n_1, n_2)\log(n_1 + n_2)}{q}}}}
\def \lemInfDelta {2(n_1 n_2)^{1-\beta}}
\def \lemInfSingleDelta {2(n_1 n_2)^{-\beta}}
\def \lemInfPLowerBound {\frac{8 k_0\mu_0 r\beta}{3(1/2 - \rhoInfty)^2} \frac{(n_1 + n_2)\log(n_1 n_2)}{n_1 n_2}}
\def \lemInfQLowerBound {{\frac{8 \mu_0 r \beta}{3 (1/2 - \rhoInfty)^2}\frac{(n_1 + n_2) \log(n_1 n_2)}{n_1 n_2}}}
\def \lemInfBetaMin {{1 + (\log 2 / \log(n_1 n_2))}}
\def \DTrThmBound {{C_0 \frac{2 \muG^2 \beta_2}{p} \sqrt{\frac{p k(n_1 + n_2)+ k\log(n_1 + n_2)}{n_1 n_2}}}}
\def \DTrThmBoundTwo {{\sqrt{C_0 \frac{2 \muG^2 \beta_2}{p}} \left(\frac{p k(n_1 + n_2)+ k\log(n_1 + n_2)}{n_1 n_2}\right)^{\frac{1}{4}}}}
\def \DTrLemMainOne {{C_1 8^h (1+\sqrt{6})^h \muG^{2h} \beta_2^h (k n_1 n_2)^{h/2} \left({p (n_1 + n_2)+ \log(n_1 + n_2)}\right)^{h/2}}}
\def \DTrLemMainTwo {{C_0\muG^2\beta_2 (k n_1 n_2)^{1/2}\sqrt{p (n_1 + n_2)+ \log(n_1 + n_2)}}}
\date{}
\title{Clipped Matrix Completion: A Remedy for Ceiling Effects}
\begin{document}

\maketitle
\global\csname @topnum\endcsname 0
\global\csname @botnum\endcsname 0
\begin{abstract}
We consider the problem of recovering a low-rank matrix from its clipped observations.
Clipping is conceivable in many scientific areas that obstructs statistical analyses.
On the other hand, \emph{matrix completion} (MC) methods can recover a low-rank matrix from various information deficits by using the principle of low-rank completion.
However, the current theoretical guarantees for low-rank MC do not apply to clipped matrices, as the deficit depends on the underlying values.
Therefore, the feasibility of \emph{clipped matrix completion} (CMC) is not trivial.
In this paper, we first provide a theoretical guarantee for the exact recovery of CMC by using a trace-norm minimization algorithm.
Furthermore, we propose practical CMC algorithms by extending ordinary MC methods. Our extension is to use the squared \emph{hinge} loss in place of the squared loss for reducing the penalty of over-estimation on clipped entries.
We also propose a novel regularization term tailored for CMC. It is a combination of two trace-norm terms, and we theoretically bound the recovery error under the regularization.
We demonstrate the effectiveness of the proposed methods through experiments using both synthetic and benchmark data for recommendation systems.
\end{abstract}
\section{Introduction}
\label{sec:org6fe9af3}
\emph{Ceiling effect} is a measurement limitation that occurs when the highest possible score on a measurement instrument is reached, thereby decreasing the likelihood that the instrument has accurately measured in the intended domain \citep{SalkindEncyclopedia2010}.
In this paper, we investigate methods for restoring a matrix data from ceiling effects.
\subsection{Ceiling effect}
\label{sec:org2e530d1}
Ceiling effect has long been discussed across a wide range of scientific fields such as sociology \citep{DeMarisRegression2004}, educational science \citep{KaplanCeiling1992,Benjaminceiling2005}, biomedical research \citep{AustinType2003,CoxAnalysis1984}, and health science \citep{Austinuse2000,CatherineRevised2004,VoutilainenHow2016,RodriguesImpact2013}, because it is a crucial information deficit known to inhibit effective statistical analyses \citep{AustinType2003}.

The ceiling effect is also conceivable in the context of machine learning, e.g., in recommendation systems with a five-star rating.
After rating an item with a five-star, a user may find another item much better later.
In this case, the true rating for the latter item should be above five, but the recorded value is still a five-star.
As a matter of fact, we can observe right-truncated shapes indicating ceiling effects in the histograms of well-known benchmark data sets for recommendation systems, as shown in Figure~\ref{fig:ceiling-effects-exist}.

Restoring data from ceiling effects can lead to benefits in many fields.
For example, in biological experiments to measure the adenosine triphosphate (ATP) level, it is known that the current measurement method has a technical upper bound \citep{YaginumaDiversity2014}.
In such a case, by measuring multiple cells in multiple environments, we may recover the true ATP levels which can provide us with further findings.
In the case of recommendation systems, we may be able to find latent superiority or inferiority between items with the highest ranking and recommend unobserved entries better.

In this paper, we investigate methods for restoring a matrix data from ceiling effects.
In particular, we consider the recovery of a clipped matrix, i.e., elements of the matrix are clipped at a predefined threshold in advance of observation, because ceiling effects are often modeled as a clipping phenomenon \citep{AustinType2003}.
\begin{figure}[t]
\begin{minipage}[c]{1.0\linewidth}
  \begin{minipage}[c]{1.0\linewidth}
    \begin{minipage}[c]{1.0\linewidth}
      \begin{minipage}[c]{0.48\linewidth}
      \begin{tikzpicture}
        \node (img)  {
          \includegraphics[keepaspectratio, width=0.85\linewidth]{\pathVisual/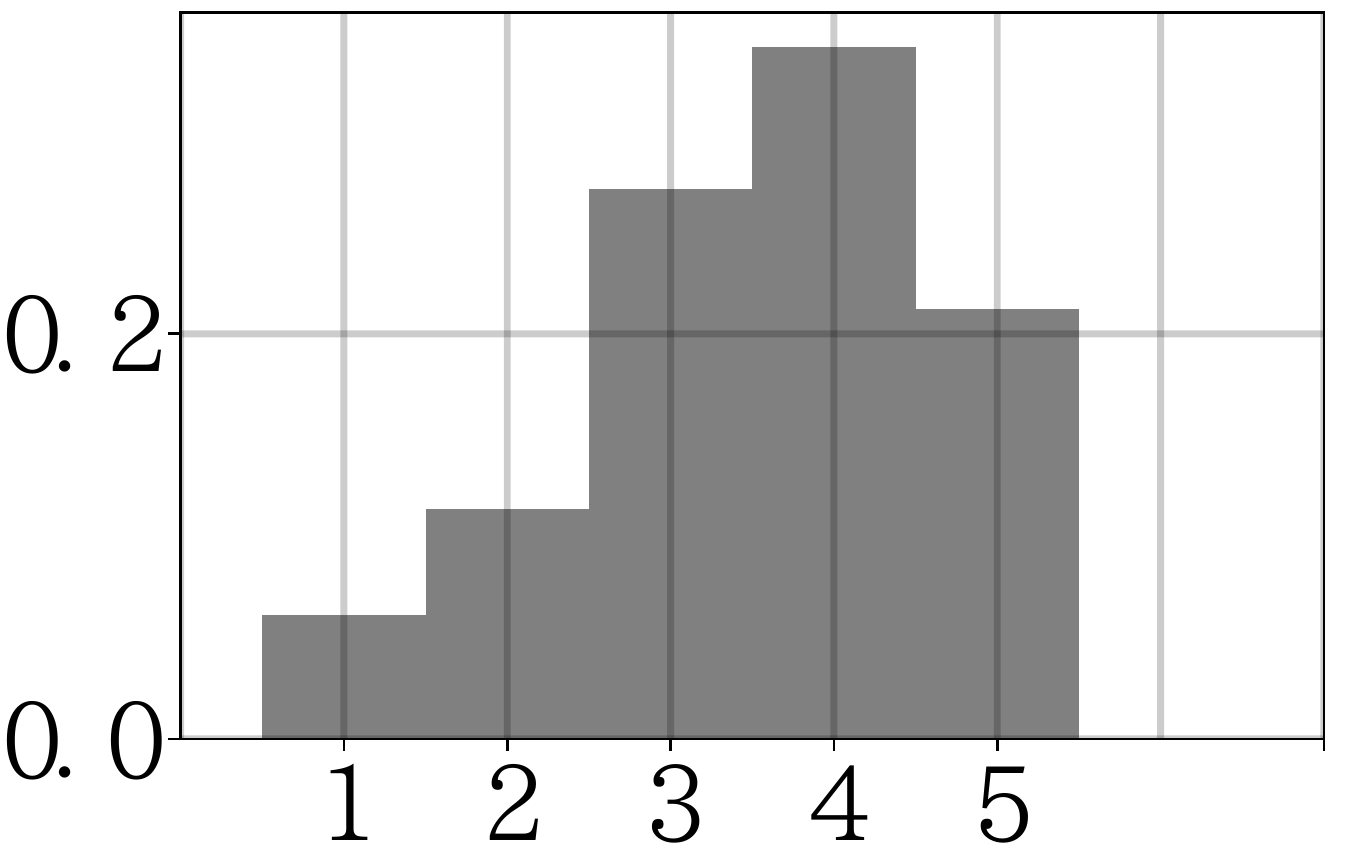}
        };
        \node[above=of img, node distance=0cm, yshift=-1.1cm,font=\color{black}] {Movielens 100K};
        \node[below=of img, node distance=0cm, yshift=1cm,font=\color{black}] {Rating};
        \node[left=of img, node distance=0cm, rotate=90, anchor=center,yshift=-0.8cm,font=\color{black}] {Probability};
      \end{tikzpicture}
      \end{minipage}\hfill
      \begin{minipage}[c]{0.48\linewidth}
      \begin{tikzpicture}
        \node (img)  {
          \includegraphics[keepaspectratio, width=0.85\linewidth]{\pathVisual/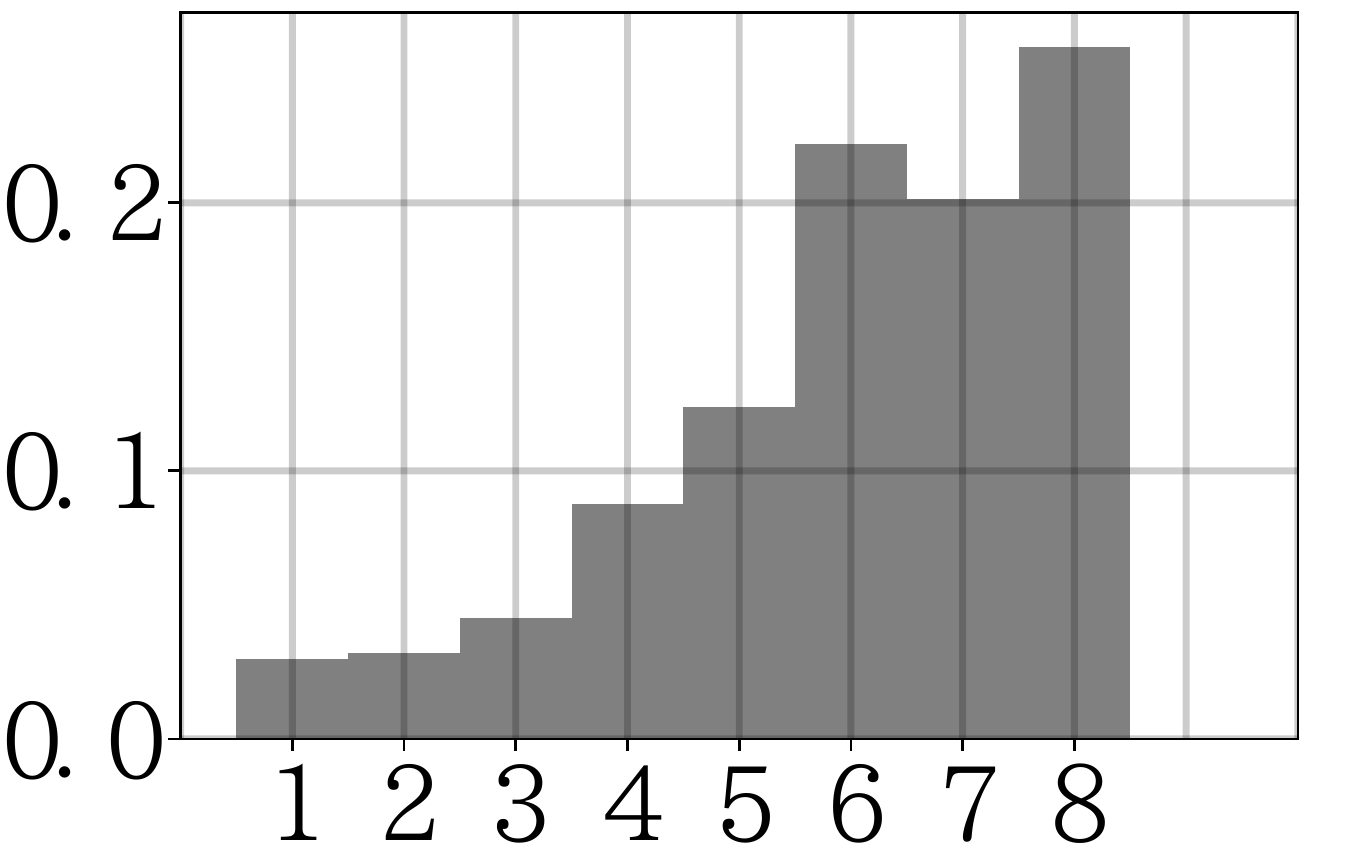}
        };
        \node[above=of img, node distance=0cm, yshift=-1.1cm,font=\color{black}] {FilmTrust};
        \node[below=of img, node distance=0cm, yshift=1cm,font=\color{black}] {Rating};
        \node[left=of img, node distance=0cm, rotate=90, anchor=center,yshift=-0.8cm,font=\color{black}] {Probability};
      \end{tikzpicture}
      \end{minipage}\hfill
    \end{minipage}
  \end{minipage}
  \begin{minipage}[c]{1.0\linewidth}
    \caption{
      Ceiling effects may also exist in standard benchmark data sets of recommendation systems (details of the data are described in Section~\ref{sec:experiments}). Histograms of the rated values are plotted. The right-truncated look of the histogram is typical for a variable under ceiling effects \citep{GreeneEconometric2012}.
    } \label{fig:ceiling-effects-exist}
  \end{minipage}
\end{minipage}
\end{figure}
\begin{figure}[t]
\begin{minipage}[c]{1.0\linewidth}
  \begin{minipage}[c]{0.3\linewidth}
    \includegraphics[keepaspectratio, width=\linewidth]{\pathIllust/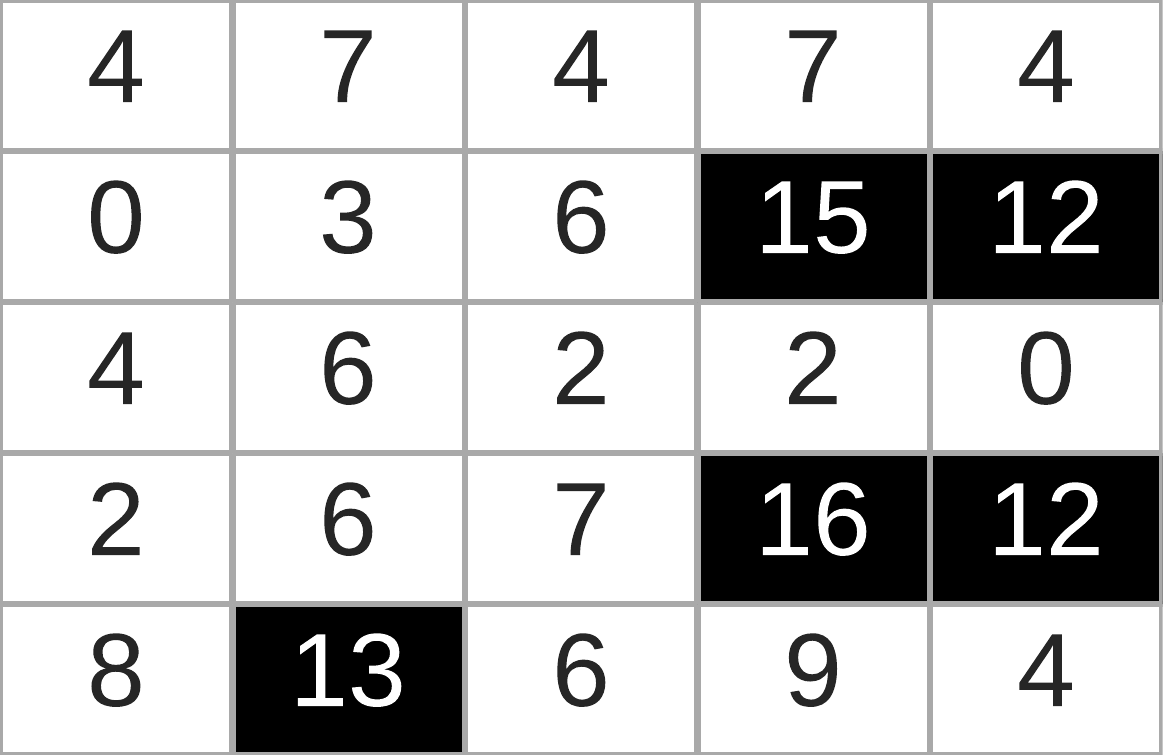}
    \subcaption{True matrix $\M$}\label{fig:res-illust-test}
  \end{minipage}\hfill
  \begin{minipage}[c]{0.3\linewidth}
    \includegraphics[keepaspectratio, width=\linewidth]{\pathIllust/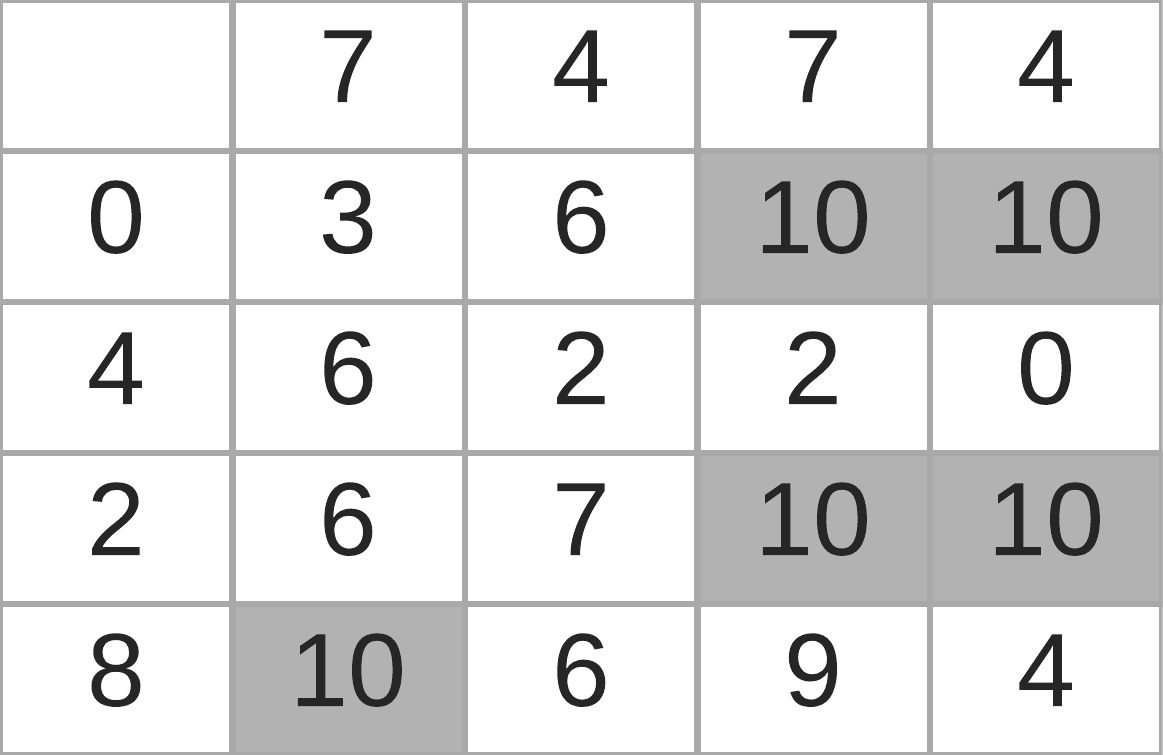}
    \subcaption{Observed $\McOmega$}\label{fig:res-illust-train}
  \end{minipage}\hfill
  \begin{minipage}[c]{0.3\linewidth}
    \includegraphics[keepaspectratio, width=\linewidth]{\pathIllust/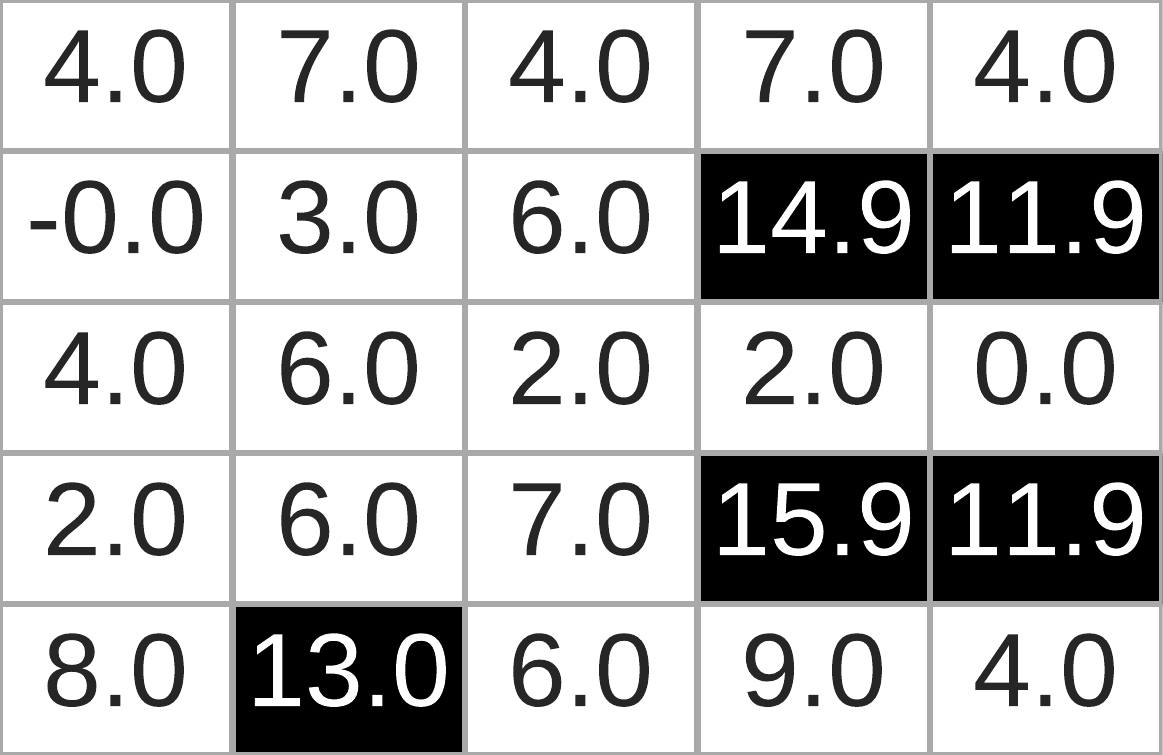}
    \subcaption{Restored $\hM$}\label{fig:res-illust-CMC}
  \end{minipage}
\end{minipage}
\begin{minipage}[c]{1.0\linewidth}
\caption{
  Illustration of the task of CMC.
  The true low-rank matrix $\M$ has a distinct structure of large values.
  However, the observed data $\McOmega$ is clipped at a predefined threshold $C=10$.
  The goal of CMC is to restore $\M$ from the value of $C$ and $\McOmega$.
  The restored matrix $\hM$ is an actual result of applying a proposed method (Fro-CMC).
  }
\label{fig:illustration}
\end{minipage}
\end{figure}
\subsection{Our problem: clipped matrix completion (CMC)}
\label{sec:org02e7593}
We consider the recovery of a low-rank matrix whose observations are clipped at a predefined threshold (Figure~\ref{fig:illustration}).
We call this problem \emph{clipped matrix completion} (CMC).
Let us first introduce its background, low-rank \emph{matrix completion}.

Low-rank \emph{matrix completion} (MC) aims to recover a low-rank matrix from various information deficits, e.g., missing \citep{CandesExact2009,Rechtsimpler2011a,ChenCompleting2015,Kiralyalgebraic2015}, noise \citep{CandesMatrix2010}, or discretization \citep{Davenport1bit2014,LanMatrix2014,BhaskarProbabilistic2016}.
The principle to enable low-rank MC is the dependency among entries of a low-rank matrix; each element can be expressed as the inner product of latent feature vectors of the corresponding row and column.
With the principle of low-rank MC, we may be able to recover the entries of a matrix from a ceiling effect.
\subsubsection{Clipped matrix completion (CMC).}
\label{sec:orgcf9a0da}
The CMC problem is illustrated in Figure~\ref{fig:illustration}.
It is a problem to recover a low-rank matrix from random observations of its entries.

Formally, the goal of CMC in this paper can be stated as follows.
Let \(\M \in \mathbb{R}^{n_1 \times n_2}\) be the ground-truth low-rank matrix where \(n_1, n_2 \in \mathbb{N}\), and \(C \in \mathbb{R}\) be the clipping threshold.
Let \(\Cf(\cdot) := \min\{C, \cdot\}\) be the clipping operator that operates on matrices element-wise.
We observe a random subset of entries of \(\Mc := \Cf(\M)\). The set of observed indices is denoted by \(\Omega\).
The goal of CMC is to accurately recover \(\M\) from \(\McOmega := \{M^\clipped_{ij}\}_{(i,j)\in\Omega}\) and \(C\).
\subsubsection{Limitations of MC.}
\label{sec:orgf489e6d}
There are two limitations regarding the application of existing MC methods to CMC.
\begin{enumerate}
\item The applicability of the principle of low-rank MC to clipped matrices is non-trivial because clipping occurs depending on the underlying values whereas the existing theoretical guarantees of MC methods presume the information deficit (e.g., missing or noise) to be independent of the values \citep{BhojanapalliUniversal2014,ChenCompleting2015,Liunew2017}.
\item Most of the existing MC methods fail to take ceiling effects into account, as they assume that the observed values are equal to or close to the true values \citep{CandesExact2009,KeshavanMatrix2010}, whereas clipped values may have a large gap from the original values.
\end{enumerate}
The goal of this paper is to overcome these limitations and to propose low-rank completion methods suited for CMC.
\subsection{Our contribution and approach}
\label{sec:orgcd4c41f}
From the perspective of MC research, our contribution is three-fold.
\subsubsection{1) We provide a theoretical analysis to establish the validity of the low-rank principle in CMC (Section~\ref{sec:feasibility}).}
\label{sec:org139c76f}
To do so, we provide an exact recovery guarantee: a sufficient condition for a trace-norm minimization algorithm to perfectly recover the ground truth matrix with high probability.
Our analysis is based on the notion of \emph{incoherence} \citep{CandesExact2009,Rechtsimpler2011a,ChenCompleting2015}.
\subsubsection{2) We propose practical algorithms for CMC (Section~\ref{sec:algorithms}) and provide an analysis of the recovery error (Section~\ref{sec:theories}).}
\label{sec:org5d91883}
We propose practical CMC methods which are extensions of the Frobenius norm minimization that is frequently used for MC \citep{Tohaccelerated2010}.
The simple idea of extension is to replace the squared loss function with the squared hinge loss to cancel the penalty of over-estimation on clipped entries.
We also propose a regularizer consisting of two trace-norm terms, which is motivated by a theoretical analysis of a recovery error bound.
\subsubsection{3) We conducted experiments using synthetic and real-world data to demonstrate the validity of the proposed methods (Section~\ref{sec:experiments}).}
\label{sec:org373833f}
Using synthetic data with known ground truth, we confirmed that the proposed CMC methods can actually recover randomly-generated matrices from clipping.
We also investigated the improved robustness of CMC methods to the existence of clipped entries in comparison with ordinary MC methods.
Using real-world data, we conducted two experiments to validate the effectiveness of the proposed CMC methods.
\subsection{Additional notation}
\label{sec:org1707aaa}
The symbols \(\M, \Mc, \McOmega, \Omega, C\), and \(\Cf\) are used throughout the paper. Let \(r\) be the rank of \(\M\).
The set of observed clipped indices is \(\mathcal{C} := \{(i, j) \in \Omega: M^\clipped_{ij} = C\}\).
Given a set of indices \(\mathcal{S}\), we define its projection operator \(\mathcal{P}_\mathcal{S} : \mathbb{R}^{n_1 \times n_2} \to \mathbb{R}^{n_1 \times n_2}\) by \((\mathcal{P}_\mathcal{S}(\X))_{ij} := \Indicator\{(i, j) \in \mathcal{S}\}X_{ij}\), where \(\Indicator\{\cdot\}\) denotes the indicator function giving \(1\) if the condition is true and \(0\) otherwise.
We use \(\|\cdot\|, \|\cdot\|_\trnrm, \|\cdot\|_\op, \|\cdot\|_\Fro\), and \(\|\cdot\|_\infty\) for the Euclidean norm of vectors, the trace-norm, the operator norm, the Frobenius norm, the infinity norm of matrices, respectively.
We also use \((\cdot)^\top\) for the transpose and define \([n] := \{1, 2, \ldots, n\}\) for \(n \in \mathbb{N}\).
For a notation table, please see Table~\ref{tbl:notation} in Appendix.
\section{Feasibility of the CMC problem \label{sec:feasibility}}
\label{sec:org5202a5a}
As noted earlier, it is not trivial if the principle of low-rank MC guarantees the recover of clipped matrices.
In this section, we establish that the principle of low-rank completion is still valid for some matrices by providing a sufficient condition under which an exact recovery by trace-norm minimization is achieved with high probability.

We consider a \emph{trace-norm minimization} for CMC:
\begin{equation}\label{eq:cmc-tr-min}\begin{split}
\widehat \M \in \mathop{\rm arg~min}\limits_{\X \in \matsp} \|\X\|_\trnrm \text{ s.t. } \begin{cases}
                                                                        \POminusC(\X) = \mathcal{P}_{\Omega\setminus \mathcal{C}}(\Mc), \\
                                                                        \mathcal{P}_{\mathcal{C}}(\Mc) \leq \mathcal{P}_{\mathcal{C}}(\X),
                                                                        \end{cases}
\end{split}\end{equation}
where \quotes{s.t.} stands for \quotes{subject to.}
Note that the optimization problem Eq.~\eqref{eq:cmc-tr-min} is convex, and there are algorithms that can solve it numerically \citep{LiuInteriorPoint2010}.
\subsection{Definitions and intuition of the information loss measures}
\label{sec:org70e2bc2}
Here, we define the quantities required for stating the theorem.
The quantities reflect the difficulty of recovering \(\M\), therefore the sufficient condition stated in the theorem will be that these quantities are small enough.
Let us begin with the definition of coherence that captures how much the row and column spaces of a matrix is aligned with the standard basis vectors \citep{CandesExact2009,Rechtsimpler2011a,ChenCompleting2015}.

\begin{definition}[Coherence and joint coherence {\citep{ChenCompleting2015}}]
Let \(\X \in \matsp\) have a skinny singular value decomposition \(\X = \tilde{\U} \tilde{\bSigma} \tilde{\V}^\top\).
We define
\begin{equation*}\begin{split}
\UCoherence(\X) := \max_{i \in [n_1]} \|\tilde{\U}_{i, \cdot}\|^2,\ \ \VCoherence(\X) := \max_{j \in [n_2]} \|\tilde{\V}_{j, \cdot}\|^2,
\end{split}\end{equation*}
where \(\tilde{\U}_{i, \cdot}\) (\(\tilde{\V}_{j, \cdot}\)) is the \(i\)-th (resp. \(j\)-th) row of \(\tilde{\U}\) (resp. \(\tilde{\V}\)).
Now the coherence of \(\M\) is defined by
\begin{equation*}\begin{split}
\mu_0 := \max\left\{\frac{n_1}{r}\UCoherence(\M), \frac{n_2}{r} \VCoherence(\M)\right\}.
\end{split}\end{equation*}
In addition, we define the following joint coherence:
\begin{equation*}\begin{split}
\jointCoherence := \sqrt{\frac{n_1 n_2}{r}} \|\UV\|_\infty.
\end{split}\end{equation*}
\label{def:coherence}
\end{definition}
The feasibility of CMC depends upon the amount of information that clipping can hide.
To characterize the amount of information obtained from observations of \(\M\),
we define a subspace \(T\) that is also used in the existing recovery guarantees for MC \citep{CandesExact2009}.

\begin{definition}[The information subspace of $\M$ \citep{CandesExact2009}]
Let \(\M = \U \mSigma \V^\top\) be a skinny singular value decomposition (\(\U \in \mathbb{R}^{n_1 \times r}, \bSigma \in \mathbb{R}^{r \times r}\) and \(\V \in \mathbb{R}^{n_2 \times r}\)).
We define
\begin{equation*}\begin{split}
T := \mathrm{span}\bigl(&\{\bm{u}_k \bm{y}^\top: k \in [r], \bm{y} \in \mathbb{R}^{n_2}\} \\
&\quad\cup \{\bm{x} \bm{v}_k^\top: k \in [r], \bm{x} \in \mathbb{R}^{n_1}\}\bigr),
\end{split}\end{equation*}
where \(\bm{u}_k, \bm{v}_k\) are the \(k\)-th column of \(\U\) and \(\V\), respectively.
Let \(\PT\) and \(\PTT\) denote the projections onto \(T\) and \(T^\perp\), respectively, where \(\perp\) denotes the orthogonal complement.
\label{def:characteristic-preliminaries}
\end{definition}
Using \(T\), we define the quantities to capture the amount of information loss due to clipping, in terms of different matrix norms representing different types of dependencies.
To express the factor of clipping, we define a transformation \(\PC\) on \(\matsp\) that describes the amount of information left after clipping.
Therefore, if these quantities are small, enough information for recovering \(\M\) may be preserved after clipping.
\begin{definition}[The information loss measured in various norms]
Define
\begin{equation*}\begin{split}
\rhoFro &:= \sup_{\Z \in T \setminus \{\mO\}: \|\Z\|_\F \leq \|\U\V^\top\|_\F} \frac{\|\PT \PC (\Z) - \Z\|_\F}{\|\Z\|_\F}, \\
\rhoInfty &:= \sup_{\Z \in T \setminus \{\mO\}: \|\Z\|_\infty \leq \|\U\V^\top\|_\infty} \frac{\|\PT \PC (\Z) - \Z\|_\infty}{\|\Z\|_\infty}, \\
\rhoOp &:= \sqrt{r} \jointCoherence \left(\sup_{\substack{\Z \in T \setminus \{\mO\}: \\ \|\Z\|_\op \leq \sqrt{n_1 n_2} \|\UV\|_\op}} \frac{\|\PC (\Z) - \Z\|_\op}{\|\Z\|_\op}\right),
\end{split}\end{equation*}
where the operator \(\PC: \matsp \to \matsp\) is defined by
\begin{equation*}\begin{split}
(\PC(\Z))_{ij} = \begin{cases}
           Z_{ij} & \text{ if } M_{ij} < C, \\
           \max\{Z_{ij}, 0\} & \text{ if } M_{ij} = C, \\
           0 & \text{ otherwise}.
           \end{cases}
\end{split}\end{equation*}
\label{def:p-star-lip}
\end{definition}

In addition, we define the following quantity that captures how much information of \(T\) depends on the clipped entries of \(\Mc\).
If this quantity is small, enough information of \(T\) may be left in non-clipped entries.
\begin{definition}[The importance of clipped entries for $T$]
Define
\begin{equation*}\begin{split}
\nuC := \|\PT\PnC\PT - \PT\|_\op,
\end{split}\end{equation*}
where \(\TrueNonClipped := \{(i, j) : M_{ij} < C\}\).
\label{def:nu-C}
\end{definition}

We follow \citet{ChenCompleting2015} to assume the following observation scheme.
As a result, it amounts to assuming that \(\Omega\) is a result of random sampling where each entry is observed with probability \(p\) independently.
\begin{assumption}[Assumption on the observation scheme]
Let \(p \in [0, 1]\).
Let \(k_0 := \kZeroValue\) and \(q := 1 - (1 - p)^{1 / k_0}\).
For each \(k = 1, \ldots, k_0\), let \(\Omega_k \subset \indsp\) be a random set of matrix indices that were sampled according to \(\Prob((i, j) \in \Omega_k) = q\) independently.
Then, \(\Omega\) was generated by \(\Omega = \bigcup_{k=1}^{k_0} \Omega_k\).
\label{assumption:sampling-scheme}
\end{assumption}
The need for Assumption~\ref{assumption:sampling-scheme} is technical \citep{ChenCompleting2015}. Refer to the proof in Appendix~\ref{app:thm1} for details.
\subsection{The theorem}
\label{sec:org4569fc4}
We are now ready to state the theorem.
\begin{theorem}[Exact recovery guarantee for CMC]
Assume \(\rhoFro < \frac{1}{2}, \rhoOp < \frac{1}{4}, \rhoInfty < \frac{1}{2}, \nuC < \frac{1}{2}\),
and Assumption~\ref{assumption:sampling-scheme} for some \(p \in [0, 1]\).
For simplicity of the statement, assume \(n_1, n_2 \geq 2\) and \(p \geq \frac{1}{n_1 n_2}\). If, additionally,
\begin{equation*}\begin{split}
p \geq \min\left\{1, \cAllRho \max(\jointCoherence^2, \mu_0) r f(n_1, n_2)\right\}
\end{split}\end{equation*}
is satisfied,
then the solution of Eq.~\eqref{eq:cmc-tr-min} is unique and equal to \(\M\) with probability at least \(1 - \delta\), where
\small
\begin{eqnarray*}
&\cAllRho &= \constFeasibilityOne \\
&&\qquad\quad\constFeasibilityTwo, \\
&f(n_1, n_2) &= \Order\left(\frac{(n_1 + n_2) (\log(n_1 n_2))^2}{n_1 n_2}\right), \\
&\delta &= \Order\left(\frac{\log(n_1 n_2)}{(n_1 + n_2)^2}\right). \\
\end{eqnarray*}
\normalsize
\label{thm:cmc-feasible}
\end{theorem}
The proof and the precise expressions of \(f\) and \(\delta\) are available in Appendix~\ref{app:thm1}. A more general form of Theorem~\ref{thm:cmc-feasible} allowing for clipping from below is also available in Appendix~\ref{app:floor-effect}.
The information losses (Def.~\ref{def:p-star-lip} and Def.~\ref{def:nu-C}) appear neither in the order of \(p\) nor that of \(\delta\), but they appear as coefficients and deterministic conditions.
The existence of such a deterministic condition is in accordance with the intuition that an all-clipped matrix can never be completed no matter how many entries are observed.

Note that \(p>1/(n_1n_2)\) can be safely assumed when there is at least one observation.
An intuition regarding the conditions on \(\rhoFro, \rhoOp, \rhoInfty\), and \(\nuC\) is that the singular vectors of \(\M\) should not be too aligned with the clipped entries for the recovery to be possible, similarly to the intuition for the incoherence condition in previous theoretical works such as \citet{CandesExact2009}.
\section{Practical algorithms \label{sec:algorithms}}
\label{sec:org2dfc7a1}
In this section, we introduce practical algorithms for CMC.
The trace-norm minimization (Eq.~\eqref{eq:cmc-tr-min})
is known to require impractical running time as the problem size increases from small to moderate or large \citep{Caisingular2010}.

A popular method for matrix completion is to minimize the squared error between the prediction and the observed value under some regularization \citep{Tohaccelerated2010}.
We develop our CMC methods following this approach.

Throughout this section, \(\X \in \mathbb{R}^{n_1 \times n_2}\) generally denotes an optimization variable, which may be further parametrized by \(\X = \mP \Q^\top\) (where \(\mP \in \mathbb{R}^{n_1 \times k}, \Q \in \mathbb{R}^{n_2 \times k}\) for some \(k \leq \min(n_1, n_2)\)).
Regularization terms are denoted by \(\mathcal{R}\), and regularization coefficients by \(\lambda, \lambda_1, \lambda_2 \geq 0\).

\subsubsection{Frobenius norm minimization for MC.}
\label{sec:org58117b7}
In the MC methods based on the Frobenius norm minimization \citep{Tohaccelerated2010}, we define
\begin{equation}\label{eq:mc-loss}\begin{split}
\fMC(\X) := \frac{1}{2}\|\Proj_\Omega(\M^\clipped - \X)\|_\Fro^2,
\end{split}\end{equation}
and obtain the estimator by
\begin{equation}\label{eq:general-alg-fro-orig}\begin{split}
\hM \in \mathop{\rm arg~min}\limits_{\X \in \matsp} \fMC(\X) + \mathcal{R}(\X).
\end{split}\end{equation}
The problem in using this method for CMC is that it is not robust to clipped entries as the loss function is designed under the belief that the true values are close to the observed values.
We extend this method for CMC with a simple idea.

\subsubsection{The general idea of extension.}
\label{sec:org9697a90}
The general idea of extension is \emph{not to penalize} the estimator on clipped entries when the predicted value exceeds the observed value.
Therefore, we modify the loss function to
\small
\begin{equation}\label{eq:loss-cmc}\begin{split}
\fCMC(\X) = \frac{1}{2} \|\Proj_{\Omega \setminus \mathcal{C}} (\M^\clipped - \X)\|_\Fro^2 + \frac{1}{2} \sum_{(i, j) \in \mathcal{C}} (\M^\clipped_{ij} - \X_{ij})_+^2,
\end{split}\end{equation}
\normalsize
where \((\cdot)_+^2 := (\max(0, \cdot))^2\) is the squared hinge loss, which does not penalize over-estimation. Then we obtain the estimator by
\begin{equation}\label{eq:general-alg-fro}\begin{split}
\hM \in \mathop{\rm arg~min}\limits_{\X \in \matsp} \fCMC(\X) + \Reg(\X).
\end{split}\end{equation}

From here, we discuss three designs of regularization terms for CMC. The methods are summarized in Table~\ref{tbl:algorithms}, and further details of the algorithms can be found in Appendix~\ref{app:algorithms}.
\subsubsection{Double trace-norm regularization.}
\label{sec:org7ac3fc3}
We first propose to use \(\Reg(\X) = \lambda_1 \|\X\|_\trnrm + \lambda_2 \|\Cf(\X)\|_\trnrm\).
For this method, we will conduct a theoretical analysis of the recovery error in Section~\ref{sec:theories}.
For optimization, we employ an iterative method based on subgradient descent \citep{AvronEfficient2012a}.
Even though the second term, \(\lambda_2 \|\Cf(\X)\|_\trnrm\), is a composition of a nonlinear mapping and a non-smooth convex function, we can take advantage of its simple structure to approximate it with a convex function of \(\X\) whose subgradient can be calculated for each iteration.
We refer to this algorithm as \emph{DTr-CMC} (Double Trace-norm regularized CMC).

\subsubsection{Trace-norm regularization.}
\label{sec:org21bca61}
With trace-norm regularization \(\Reg(\X) := \lambda \|\X\|_\trnrm\), the optimization problem Eq.~\eqref{eq:general-alg-fro} is a relaxation of the trace-norm minimization (Eq.~\eqref{eq:cmc-tr-min}) by replacing the exact constraints with the quadratic penalties (Eq.~\eqref{eq:mc-loss} for MC and Eq.~\eqref{eq:loss-cmc} for CMC).
For optimization, we employ an accelerated proximal gradient (APG) algorithm proposed by \citet{Tohaccelerated2010}, by taking advantage of the differentiability of the squared hinge loss.
We refer to this algorithm as \emph{Tr-CMC} (Trace-norm-regularized CMC), in contrast to \emph{Tr-MC} (its MC counterpart; \citealp{Tohaccelerated2010}).
\subsubsection{Frobenius norm regularization.}
\label{sec:org2d17138}
This method first parametrizes \(\X\) as \(\mP \Q^\top\) and use \(\mathcal{R}(\mP, \Q) := \lambda_1\|\mP\|_\Fro^2 + \lambda_2\|\Q\|_\Fro^2\) for regularization.
A commonly used method for optimization in the case of MC is the alternating least squares (ALS) method \citep{JainLowrank2013}.
Here, we employ an approximate optimization scheme motivated by ALS in our experiments.
We refer to this algorithm as \emph{Fro-CMC} (Frobenius-norm-regularized CMC), in contrast to \emph{Fro-MC} (its MC counterpart; \citealp{JainLowrank2013}).

\begin{table}
\begin{center}
\captionof{table}{\label{tbl:algorithms}
List of the proposed methods for CMC (Fro: Frobenius norm, Tr: Trace-norm, Sq.hinge: Squared hinge loss, SUGD: SUb-Gradient Descent, APG: Accelerated Proximal Gradient, ALS: Alternating Least Squares, Param.: Parametrization, Reg.: Regularization, Opt.: Optimization).}
\begin{tabular}{lllll}
\hline
Method & Param. & Loss on \(\mathcal{C}\) & Reg. & Opt.\\
\hline
DTr-CMC & \(\X\) & Sq. hinge & Tr + Tr & SUGD\\
Tr-CMC & \(\X\) & Sq. hinge & Tr & APG\\
Fro-CMC & \(\mP\Q^\top\) & Sq. hinge & Fro & ALS\\
\hline
\end{tabular}
\end{center}
\end{table}
\section{Theoretical analysis for DTr-CMC \label{sec:theories}}
\label{sec:org7de0300}
In this section, we provide a theoretical guarantee for DTr-CMC.
Let \(G\) be the hypothesis space defined by
\begin{equation*}\begin{split}
G = \bigl\{\X \in \matsp:& \|\X\|_\trnrm^2 \leq \beta_1 \sqrt{k n_1 n_2},\\
& \|\Cf(\X)\|_\trnrm^2 \leq \beta_2 \sqrt{k n_1 n_2}\bigr\}
\end{split}\end{equation*}
for some \(k \leq \min(n_1, n_2)\) and \(\beta_1, \beta_2 \geq 0\).
Here, we analyze the estimator
\begin{equation}\label{eq:dtr-algorithm}\begin{split}
\widehat \M \in \mathop{\rm arg~min}\limits_{\X \in G} \sum_{(i, j) \in \Omega} (M^\clipped_{ij} - \Cf(X_{ij}))^2.
\end{split}\end{equation}
The minimization objective of Eq.~\eqref{eq:dtr-algorithm} is not convex. However, it is upper-bounded by the convex loss function \(\fCMC\) (Eq.~\eqref{eq:loss-cmc}). The proof is provided in Appendix~\ref{app:algorithms:dtr-cmc}.
Therefore, DTr-CMC can be seen as a convex relaxation of Eq.~\eqref{eq:dtr-algorithm} with constraints turned into regularization terms.
To state our theorem, we define the unnormalized coherence of a matrix.
\begin{definition}[Unnormalized coherence]
Let \(\unnCoherence(\X)\) be unnormalized coherence defined by
\begin{eqnarray*}
\unnCoherence(\X)=\max\{\UCoherence(\X), \VCoherence(\X)\},
\end{eqnarray*}
using \(\UCoherence\) and \(\VCoherence\) from Def.~\ref{def:coherence}.
\end{definition}

Now we are ready to state our theorem.
\begin{theorem}[Theoretical guarantee for DTr-CMC]
Suppose that \(\M \in G\), and that \(\Omega\) is generated by independent observation of entries with probability \(p \in [0, 1]\).
Let \(\muG = \supG \unnCoherence(\Cf(\X))\), and \(\widehat \M\) be a solution to the optimization problem Eq.~\eqref{eq:dtr-algorithm}.
Then there exist universal constants \(C_0\) and \(C_1\), such that with probability at least \(1 - C_1/(n_1 + n_2)\) we have
\begin{equation}\label{eq:thm-dtr-bound}\begin{split}
&\sqrt{\frac{1}{n_1 n_2}\|\widehat \M - \M\|_\Fro^2} \\
&\leq \underbrace{\frac{\|\M - \Mc\|_\Fro}{\sqrt{n_1 n_2}}}_{=B_1: \text{Complexity of data}} + \underbrace{\frac{\|\hM - \Cf(\hM)\|_\Fro}{\sqrt{n_1 n_2}}}_{=B_2: \text{Complexity of hypothesis}} \\
&\qquad + \underbrace{\frac{\|\mathrm{Clip}(\widehat \M) - \mathrm{Clip}(\M)\|_\Fro}{\sqrt{n_1 n_2}}}_{=B_3: \text{Estimation error}},
\end{split}\end{equation}
and
\begin{equation*}\begin{split}
B_1 &\leq (\sqrt{\beta_1} + \sqrt{\beta_2}) k^{\frac{1}{4}} (n_1 n_2)^{-\frac{1}{4}}, \\
B_2 &\leq (\sqrt{\beta_1} + \sqrt{\beta_2}) k^{\frac{1}{4}} (n_1 n_2)^{-\frac{1}{4}}, \\
B_3 &\leq \DTrThmBoundTwo.
\end{split}\end{equation*}
\label{thm:dtr-guarantee}
\end{theorem}
We provide the proof in Appendix~\ref{app:thm2}.
The right-hand side of Eq.~\eqref{eq:thm-dtr-bound} converges to zero as \(n_1, n_2 \to \infty\) with \(p, k, \beta_1\), and \(\beta_2\) fixed.
From this theorem, it is expected that if \(\|\M\|_\trnrm\) and \(\|\M^\c\|_\trnrm\) are believed to be small, DTr-CMC can accurately recover \(\M\).
\section{Related work \label{sec:related-works}}
\label{sec:org4265e24}
In this section, we describe related work from the literature on matrix completion and that on ceiling effects. Table~\ref{tbl:relative-position} briefly summarizes the related work on matrix completion.
\subsection{Matrix completion methods.}
\label{sec:org560f7e5}
\subsubsection{Theory.}
\label{sec:org6d2459c}
Our feasibility analysis in Section~\ref{sec:feasibility} followed the approach of \citet{Rechtsimpler2011a} while some details of the proof were based on \citet{ChenCompleting2015}.
There is further research to weaken the assumption of the uniformly random observation \citep{BhojanapalliUniversal2014}.
It may be relatively easy to incorporate such devices into our theoretical analysis.

Our theoretical analysis for DTr-CMC in Section~\ref{sec:theories} is inspired by the theory for 1-bit matrix completion \citep{Davenport1bit2014}.
The difference is that our theory effectively exploits the additional low-rank structure in the clipped matrix in addition to the original matrix.
\subsubsection{Problem setting.}
\label{sec:org91ee0a3}
Our problem setting of clipping can be related to quantized matrix completion (Q-MC; \citealp{LanMatrix2014}; \citealp{BhaskarProbabilistic2016}).
\citet{LanMatrix2014} and \citet{BhaskarProbabilistic2016} formulated a probabilistic model which assigns discrete values according to a distribution conditioned on the underlying values of a matrix.
\citet{BhaskarProbabilistic2016} provided an error bound for restoring the underlying values, assuming that the quantization model is fully known.
The model of Q-MC can provide a different formulation for ceiling effects from ours by assuming the existence of latent random variables.
However, Q-MC methods require the data to be fully discrete \citep{LanMatrix2014,BhaskarProbabilistic2016}.
Therefore, neither their methods nor theories can be applied to real-valued observations.
On the other hand, our methods and theories allow observations to be real-valued.
The ceiling effect is worth studying independently of quantization, since the data analyzed under ceiling effects are not necessarily discrete.
\subsubsection{Methodology.}
\label{sec:orgdb4a391}
The use of the Frobenius norm for MC has been studied for MC from noisy data \citep{CandesMatrix2010,Tohaccelerated2010}.
Our algorithms are based on this line of research, while extending it for CMC.

Methodologically, \citet{MarecekMatrix2017} is closely related to our Fro-CMC.
They considered completion of missing entries under \quotes{interval uncertainty} which yields interval constraints indicating the ranges in which the true values should reside.
They employed the squared hinge loss for enforcing the interval constraints in their formulation, hence coinciding with our formulation of Fro-CMC.
There are a few key differences between their work and ours.
First, our motivations are quite different as we are analyzing a different problem from theirs. They considered completion of missing entries with robustness to uncertainty, whereas we considered recovery of clipped entries.
Secondly, they did not provide any theoretical analysis of the problem. We provided an analysis by specifically investigating the problem of clipping.
Lastly, as a minor difference, we employed an ALS-like algorithm whereas they used a coordinate descent method \citep{MarecekMatrix2017,MarecekLowRank2018}, as we found the ALS-like method to work well for moderately sized matrices.
\subsection{Related work on ceiling effects}
\label{sec:orgb929813}
From the perspective of dealing with ceiling effects, the present paper adds a potentially effective method to the analysis of data affected by a ceiling effect.
Ceiling effect is also referred to as \emph{censoring} \citep{GreeneEconometric2012} or \emph{limited response variables} \citep{DeMarisRegression2004}.
In this paper, we use \quotes{ceiling effect} to represent these phenomena.
In \emph{econometrics}, Tobit models are used to deal with ceiling effects \citep{GreeneEconometric2012}. In Tobit models, a \emph{censored likelihood} is modeled and maximized with respect to the parameters of interest.
Although this method is justified by the theory of M-estimation \citep{SchnedlerLikelihood2005,GreeneEconometric2012}, its use for matrix completion is not justified.
In addition, Tobit models require strong distributional assumptions, which is problematic especially if the distribution cannot be safely assumed.
\begin{table}[t]
\caption{\label{tbl:relative-position}
Our target problem is the restoration of a low-rank matrix from clipping at a predefined threshold. No existing work has considered this type of information deficit.}
\centering
\begin{tabular}{ll}
\hline
Type of deficit & Related work\\
\hline
Missing & \cite{CandesExact2009} etc.\\
Noise & \cite{CandesMatrix2010} etc.\\
Quantization & \cite{BhaskarProbabilistic2016} etc.\\
Clipping & This paper\\
\hline
\end{tabular}
\end{table}
\section{Experimental results \label{sec:experiments}}
\label{sec:org465ff76}
In this section, we show the results of experiments to compare the proposed CMC methods to the MC methods.
\subsection{Experiment with synthetic data}
\label{sec:orgba47ea4}
We conducted an experiment to recover randomly generated data from clipping.
The primary purpose of the experiment was to confirm that the principle of low-rank completion is still effective for the recovery of a clipped matrix, as indicated by Theorem~\ref{thm:cmc-feasible}.
Additionally, with the same experiment, we investigated how sensitive the MC methods are to the clipped entries by looking at the growth of the recovery error in relation to increased rates of clipping.
\subsubsection{Data generation process.}
\label{sec:org5ea2087}
We randomly generated non-negative integer matrices of size \(500 \times 800\) that are exactly rank-\(30\) with the fixed magnitude parameter \(L = 15\) (see Appendix~\ref{app:data-generation}).
The generated elements of matrix \(\M\) were randomly split into three parts with ratio \((0.8, 0.1, 0.1)\). Then the first part was clipped at the threshold \(C\) (varied over \(\{5, 6, 7, 8, 9, 11, 13\}\)) to generate the training matrix \(\McOmega\) (therefore, \(p=0.8\)). The remaining two parts (without thresholding) were treated as the validation (\(\M^\valid\)) and testing (\(\M^\test\)) matrices, respectively.
\subsubsection{Evaluation metrics.}
\label{sec:org343112c}
We used the relative root mean square error (rel-RMSE) as the evaluation metric, and we considered a result as a good recovery if the error is of order \(10^{-2}\) \citep{Tohaccelerated2010}.
We separately reported the rel-RMSE on two sets of indices: all the indices of \(\M\), and the test entries whose true values are below the clipping threshold.
For hyperparameter tuning, we used the rel-RMSE after clipping on validation indices: \(\frac{\|\Cf(\hM) - \Cf(\M^\valid)\|_\Fro}{\|\Cf(\M^\valid)\|_\Fro}\).
We reported the mean of five independent runs.
The clipping rate was calculated by the ratio of entries of \(\M\) above \(C\).
\subsubsection{Compared methods.}
\label{sec:orge46b329}
We evaluated the proposed methods (DTr-CMC, Tr-CMC, and Fro-CMC) and their MC counterparts (Tr-MC and Fro-MC).
We also applied MC methods after ignoring all clipped training entries (Tr-MCi and Fro-MCi, with \quotes{i} standing for \quotes{ignore}).
While this treatment wastes some data, it may improve the robustness of MC methods to the existence of clipped entries.
\subsubsection{Result 1: The validity of low-rank completion.}
\label{sec:org62d576e}
In Figure~(\ref{fig:exp:synth}\subref{fig:experiment:synthetic:1}), we show the rel-RMSE for different clipping rates.
The proposed methods successfully recover the true matrices with very low error of order \(10^{-2}\) even when half of the observed training entries are clipped. One of them (Fro-CMC) is able to successfully recover the matrix after the clipping rate was above \(0.6\).
This may be explained in part by the fact that the synthetic data were exactly low rank, and that the correct rank was in the search space of the bilinear model of the Frobenius norm based methods.
\subsubsection{Result 2: The robustness to the existence of clipped training entries.}
\label{sec:org0cc8482}
In Figure~(\ref{fig:exp:synth}\subref{fig:experiment:synthetic:2}), the recovery error of MC methods on non-clipped entries increased with the rate of clipping.
This indicates the disturbance effect of the clipped entries for ordinary MC methods.
The MC methods with the clipped entries ignored (Tr-MCi and Fro-MCi) were also prone to increasing test error on non-clipped entries for high clipping rates, most likely due to wasting too much information.
On the other hand, the proposed methods show improved profiles of growth, indicating improved robustness.
\begin{figure}[t]
  \begin{minipage}[c]{1.0\linewidth}
  \begin{minipage}[c]{1.0\linewidth}
    \begin{minipage}[c]{0.95\linewidth}
      \begin{minipage}[c]{0.5\linewidth}\hfill
          \begin{tikzpicture}
            \node (img)  {
              \includegraphics[keepaspectratio, width=0.9\linewidth,pagebox=cropbox]{\pathSynth/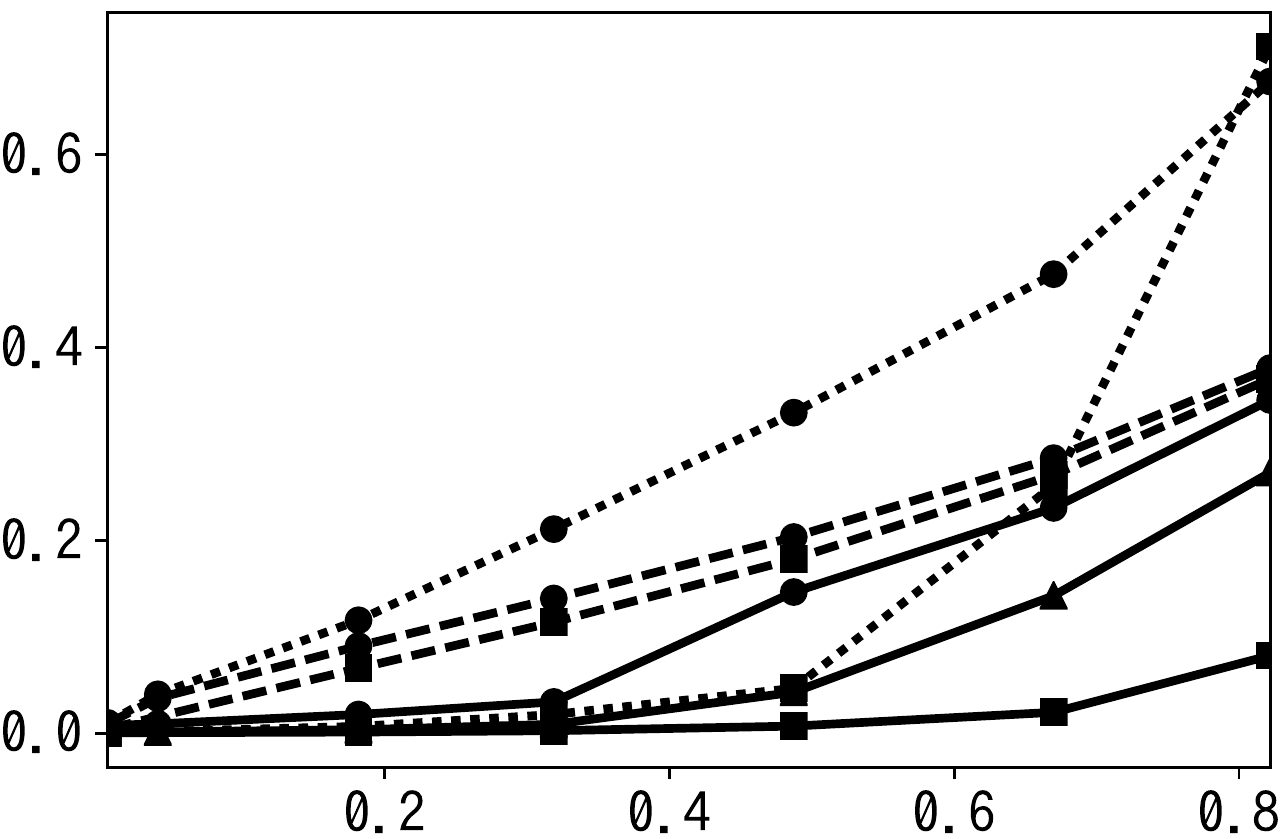}
            };
            \node[below=of img, node distance=0cm, yshift=1cm,font=\color{black}] {Clipping rate};
            \node[left=of img, node distance=0cm, rotate=90, anchor=center,yshift=-0.9cm,font=\color{black}] {Rel-RMSE};
          \end{tikzpicture}
          \subcaption{On all entries}\hfill
          \label{fig:experiment:synthetic:1}
      \end{minipage}\hfill
      \begin{minipage}[c]{0.5\linewidth}\hfill
          \begin{tikzpicture}
            \node (img)  {
              \includegraphics[keepaspectratio, width=0.9\linewidth,pagebox=cropbox]{\pathSynth/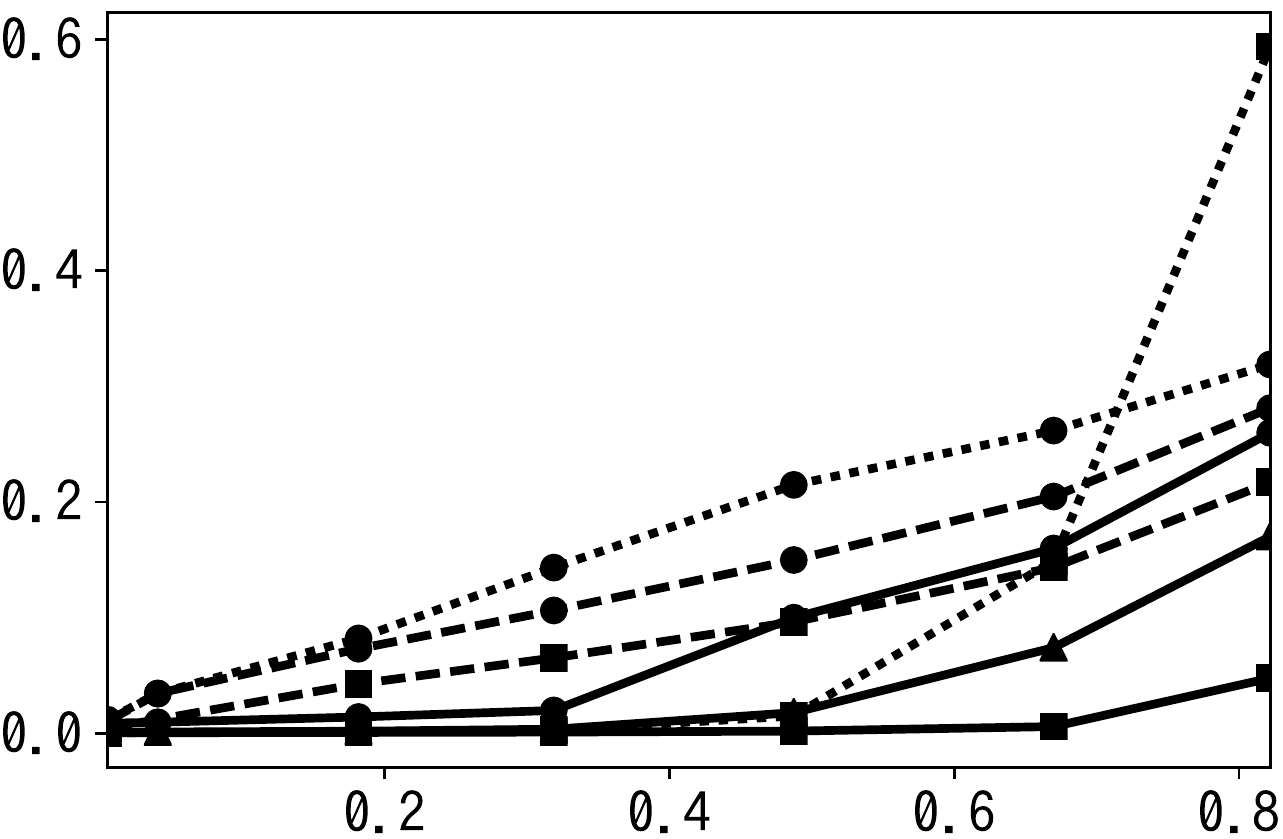}
            };
            \node[below=of img, node distance=0cm, yshift=1cm,font=\color{black}] {Clipping rate};
            \node[left=of img, node distance=0cm, rotate=90, anchor=center,yshift=-0.9cm,font=\color{black}] {Rel-RMSE};
          \end{tikzpicture}
          \subcaption{On non-clipped test entries}\hfill
          \label{fig:experiment:synthetic:2}
   \end{minipage}
    \end{minipage}
  \end{minipage}
  \begin{minipage}[c]{1.0\linewidth}
    \begin{minipage}[c]{0.5\linewidth}
      \caption{Relative RMSE for varied $C$ (Dotted: previous MC methods, Solid: proposed CMC methods).}
      \label{fig:exp:synth}
    \end{minipage}\hfill
    \begin{minipage}[c]{0.5\linewidth}\hfill
      \includegraphics[keepaspectratio, width=0.95\linewidth,pagebox=cropbox]{\pathSynth/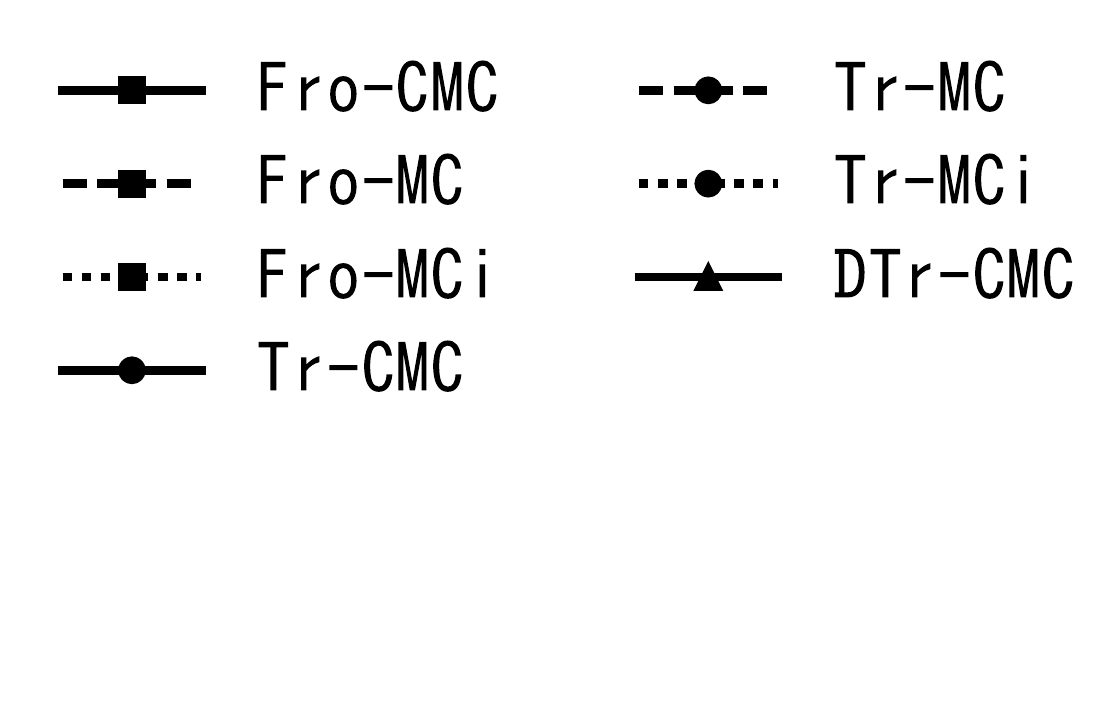}
    \hfill\end{minipage}
  \end{minipage}
\end{minipage}
\end{figure}
\subsection{Experiments with real-world data}
\label{sec:org486f3dd}
We conducted two experiments using real-world data.
The difficulty of evaluating CMC with real-world data is that there are no known true values unaffected by the ceiling effect.
Therefore, instead of evaluating the accuracy of recovery, we evaluated the performance of distinguishing entries with the ceiling effect and those without.
We considered two binary classification tasks in which we predict whether held-out test entries are of high values.
The tasks would be reasonable because the purpose of a recommendation system is usually to predict which entries have \emph{high} scores.
\subsubsection{Preparation of data sets.}
\label{sec:orgddb938a}
We used the following benchmark data sets of recommendation systems.
\begin{itemize}
\item FilmTrust \citep{GuoNovel2013}\footnote{\url{https://www.librec.net/datasets.html}} consists of ratings obtained from 1,508 users to 2,071 movies on a scale from \(0.5\) to \(4.0\) with a stride of \(0.5\) (approximately 99.0\% missing). For ease of comparison, we doubled the ratings so that they are integers from \(1\) to \(8\).
\item Movielens (100K)\footnote{\url{http://grouplens.org/datasets/movielens/100k/}} consists of ratings obtained from 943 users to 1,682 movies on an integer scale from \(1\) to \(5\) (approximately 94.8\% missing).
\end{itemize}
\subsubsection{Task 1: Using artificially clipped training data.}
\label{sec:org8e67db9}
We artificially clipped the training data at threshold \(C\) and predicted whether the test entries were originally above \(C\).
We used \(C = 7\) for FilmTrust and \(C = 4\) for Movielens.
For testing, we made positive prediction for entries above \(C+0.5\) and negative prediction otherwise.
\subsubsection{Task 2: Using raw data.}
\label{sec:org38de865}
We used the raw training data and predicted whether the test entries are equal to the maximum value of the rating scale (i.e., the underlying values are at least the maximum value).
For CMC methods, we set \(C\) to the maximum value, i.e., \(C = 8\) for FilmTrust and \(C = 5\) for Movielens.
For testing, we made positive prediction for entries above \(C-0.5\) and negative prediction otherwise.
\subsubsection{Protocols and evaluation metrics.}
\label{sec:org71e6be4}
In both experiments, we first split the observed entries into three groups with ratio \((0.8, 0.1, 0.1)\), which were used as training, validation, and test entries.
Then for the first task, we artificially clipped the training data at \(C\).
If a user or an item had no training entries, we removed them from all matrices.

We measured the performance by the \(\fone\) score.
Hyperparameters were selected by the \(\fone\) score on the validation entries.
We reported the mean and the standard error after five independent runs.
\subsubsection{Compared methods.}
\label{sec:org4fb2dca}
We compared the proposed CMC methods with the corresponding MC methods.
The baseline (indicated as \quotes{baseline}) is to make prediction positively for all entries, for which the recall is \(1\) and the precision is the ratio of the positive data. This is the best baseline in terms of the \(\fone\) score without looking at data.
\subsubsection{Results.}
\label{sec:org83ae4e5}
The results are compiled in Table~\ref{tbl:experiment:recom1and2}.
In Task 1, by comparing the results between CMC methods and their corresponding MC methods, we conclude that the CMC methods have improved the ability to recover clipped values in real-world data as well.
In Task 2, the CMC methods show better performance for predicting entries of the maximum value of rating than their MC counterparts.

Interestingly, we obtain the performance improvement by only changing the loss function to be robust to ceiling effects and without changing the model complexity (such as introducing an ordinal regression model).
\begin{table}[t]
\centering
\captionof{table}{
\label{tbl:experiment:recom1and2}
Results of the two tasks measured in $\fone$. Bold-face indicates the highest score.}
\begin{tabular}[t]{|*1{p{14mm}}|*1{p{15mm}|}|*2{p{16mm}|}}
  \hline
  Data & Methods & Task 1 \(\fone\) & Task 2 \(\fone\)\\
  \hline
  Film & DTr-CMC    &  \textbf{0.47 (0.01)}   &\textbf{0.46 (0.01)} \\
  \cline{2-4}Trust& Fro-CMC    &  0.35 (0.01)            &  0.40 (0.01)  \\
  \cline{2-4}& Fro-MC     &  0.27 (0.01)            &  0.35 (0.01) \\
  \cline{2-4}& Tr-CMC     &  0.36 (0.00)            &  0.39 (0.00) \\
  \cline{2-4}& Tr-MC      &  0.22 (0.00)            &  0.35 (0.01) \\
  \cline{2-4}& (baseline) &  0.41 (0.00)            &  0.41 (0.00) \\
  \hline
  Movielens & DTr-CMC  &  0.39 (0.00)          &  0.38 (0.00) \\
  \cline{2-4}(100K)& Fro-CMC  &  \textbf{0.41 (0.00)} &  \textbf{0.41 (0.01)} \\
  \cline{2-4}& Fro-MC         &  0.21 (0.01)          &  0.38 (0.01) \\
  \cline{2-4}& Tr-CMC         &  0.40 (0.00)          &  0.40 (0.00) \\
  \cline{2-4}& Tr-MC          &  0.12 (0.00)          &  0.38 (0.00) \\
  \cline{2-4}& (baseline)     &  0.35 (0.00)          &  0.35 (0.00) \\
  \hline
\end{tabular}
\end{table}
The computation time of the proposed methods are reported in Appendix~\ref{app:run-time}.
\section{Conclusion}
\label{sec:org39392bb}
In this paper, we showed the first result of exact recovery guarantee for the novel problem of clipped matrix completion.
We proposed practical algorithms as well as a theoretically-motivated regularization term.
We showed that the clipped matrix completion methods obtained by modifying ordinary matrix completion methods are more robust to clipped data, through numerical experiments.
A future work is to specialize our theoretical analysis to discrete data to analyze the ability of quantized matrix completion methods for recovering discrete data from ceiling effects.
\section*{Acknowledgments}
\label{sec:org4c76e6a}
We would like to thank Ikko Yamane, Han Bao, and Liyuan Xu, for helpful discussions.
MS was supported by the International Research Center for Neurointelligence (WPI-IRCN) at The University of Tokyo Institutes for Advanced Study.
\bibliographystyle{aaai}
\fontsize{9.0pt}{10.0pt} \selectfont \bibliography{2019AAAI.bib} \normalsize
\clearpage

\onecolumn

\global\csname @topnum\endcsname 0
\global\csname @botnum\endcsname 0

\begin{appendices}
Here, we use the same notation as in the main article.
Throughout, we use \(\langle X, Y \rangle = \trace(X^\top Y)\) as the inner product of two matrices \(X, Y \in \matsp\), where \(\trace\) is the trace of a matrix.
For iterative algorithms, a superscript of \((t)\) is used to indicate the quantities at iteration \(t\).
\section{Details of the algorithms \label{app:algorithms}}
\label{sec:orgb8727fc}
Here we describe instances of CMC (clipped matrix completion) algorithms in detail. A more general description can be found in Section~\ref{sec:algorithms} of the main article.
\subsection{Details of DTr-CMC \label{app:algorithms:dtr-cmc}}
\label{sec:orgb372cd5}
The optimization method for DTr-CMC based on subgradient descent \citep{AvronEfficient2012a} is described in Algorithm~\ref{alg:algorithm-DTr-CMC}.
In the algorithm, we let \(\svd\) denote a skinny singular value decomposition subroutine, and \(\mE \in \mathbb{R}^{n_1 \times n_2}\) the all-one matrix.
\begin{algorithm}
\caption{A subgradient descent algorithm for DTr-CMC}
\label{alg:algorithm-DTr-CMC}
\begin{algorithmic}
\REQUIRE $T \in \mathbb{N}$, $\epsilon > 0$, $\lambda_1$, $\lambda_2$, $\{\eta_t\}_{t=1}^T$, and $C \in \mathbb{R}$
\STATE \textbf{Initialize:} $\X^{(0)} = (C+1) \mE$
\FOR{$t = 1, \ldots, T$}
\STATE $\H^{(t)} \gets \nabla \fCMC(\X^{(t-1)})$ // Take the gradient of $f$
\STATE $\U_1^{(t)}, \bSigma_1^{(t)}, \V_1^{(t)\top} \gets \svd(\W(\X^{(t-1)}) \hadamard (\X^{(t-1)} - C) + C)$ // Take a subgradient of $\|\W(\X^{(t-1)}) \hadamard (\X - C) + C\|_\trnrm$
\STATE $\G^{(t)} \gets \W(\X^{(t-1)}) \hadamard (\U_1^{(t)} \V_1^{(t)\top})$ // Same as above
\STATE $\tilde{\X}^{(t)} \gets \X^{(t-1)} - \eta_t (\H^{(t)} + \lambda_2 \G^{(t)})$ // Update in the direction of the subgradient
\STATE $\U_2^{(t)}, \bSigma_2^{(t)}, \V_2^{(t)\top} \gets \svd(\tilde{\X^{(t)}})$
\STATE $\bSigma \gets \bSigma_2^{(t)} - \eta_t \lambda_1 \I$
\STATE $Z_{ij} \gets \Sigma_{ij} > \epsilon$
\STATE $\X^{(t)} \gets \U_2^{(t)} (\bSigma \hadamard \Z) \V_2^{(t)\top}$
\ENDFOR
\RETURN $\mathop{\rm arg~min}\limits_{\X^{(t)}: t = 0, \ldots, T} \fCMC(\X^{(t)}) + \lambda_1 \|\X^{(t)}\|_\trnrm + \lambda_2 \|\W(\X^{(t)}) \hadamard (\X^{(t)} - C) + C\|_\trnrm$
\end{algorithmic}
\end{algorithm}
\subsubsection{Derivation of the algorithm.}
\label{sec:orgd2fbf5c}
Let \(\hadamard\) denote the Hadamard product. The second regularization term of DTr-CMC can be rewritten as
\begin{equation*}\begin{split}
&\lambda_2 \|\Cf(\X)\|_\trnrm = \lambda_2 \|\X \hadamard \W(\X) + C \cdot (1 - \W(\X))\|_\trnrm = \lambda_2 \|\W(\X) \hadamard (\X - C) + C\|_\trnrm
\end{split}\end{equation*}
where \(\W(\X)\) is defined by \(W_{ij}(\X) = \Indicator\{X_{ij} < C\}\).
This is a composition of a non-smooth convex function \(\|\cdot\|_\trnrm\) and a nonlinear operator \(\Cf\), hence it is not trivial to find a method to minimize this function.
Here, in order to minimize the objective function, we employ an iterative scheme to approximate this function with one that has a known subgradient.
For each iteration \((t)\), we find a subgradient of the following heuristic objective at \(\X = \X^{(t-1)}\):
\begin{equation*}\begin{split}
\fCMC(\X) + \lambda_1 \|\X\|_\trnrm + \lambda_2 \|\W(\X^{(t-1)}) \hadamard (\X - C) + C\|_\trnrm.
\end{split}\end{equation*}
and update the parameter in the descending direction. This function is a combination of the trace-norm and a linear transformation.
A subgradient can be calculated by first performing a singular value decomposition \(\U^{(t)} \bSigma^{(t)} \V^{(t)} = \W(\X^{(t-1)}) \hadamard (\X^{(t-1)} - C) + C\), and then calculating \(\W(\X^{(t-1)}) \hadamard (\U^{(t)} \V^{(t)})\).
\subsubsection{Relation to the theoretical analysis in Section~\ref{sec:theories}.}
\label{sec:org38e4c8b}
Here, we show that the loss function \(\fCMC\) (Eq.~\eqref{eq:loss-cmc}) is a convex upper bound of the loss function in Eq.~\eqref{eq:dtr-algorithm}.
\begin{proof}
By a simple calculation,
\begin{equation*}\begin{split}
&\sum_{(i, j) \in \Omega \setminus \mathcal{C}} (M^\clipped_{ij} - X_{ij})^2 + \sum_{(i, j) \in \mathcal{C}} (M^\clipped_{ij} - X_{ij})_+^2 - \sum_{(i, j) \in \Omega} (M^\clipped_{ij} - \mathrm{Clip}(X_{ij}))^2\\
&= \sum_{(i, j) \in \Omega \setminus \mathcal{C}, X_{ij} \geq C} (M^\clipped_{ij} - X_{ij})^2 - (M^\clipped_{ij} - C)^2\\
&= \sum_{(i, j) \in \Omega \setminus \mathcal{C}, X_{ij} \geq C} (C - X_{ij}) (2 M^\clipped_{ij} - C - X_{ij}) \\
&\geq 0.
\end{split}\end{equation*}
Therefore, \(\fCMC\) is an upper bound of the objective function of Eq.~\eqref{eq:dtr-algorithm}.
\end{proof}
\subsubsection{Initialization.}
\label{sec:org1f80a6c}
While we expect the regularization term \(\|\X\|_\trnrm\) to encourage the recovery of the values above the threshold,
the task is difficult as it requires extrapolating the values to outside the range of any observed entries.
To compensate for this difficulty, we initialize the parameter matrix with values strictly above the threshold. This allows the algorithm to start from a matrix whose values are above the threshold and simplify the hypothesis.
Therefore, in the experiment, we initialized all elements of \(\X^{(0)}\) with \(C + 1\) (here, we used \(1\) reflecting the spacing between choices on the rating scale of the benchmark data of recommendation systems. This value can be arbitrarily configured).
\subsubsection{Experiments.}
\label{sec:org1914367}
In both experiments using synthetic data and those using real-world data, we used \(T \in \{1000, 2000\}, \eta_0 \in \{0.5, 1.0, 1.5\}, \eta_t = \eta_0 \cdot 0.99^{t-1}, \epsilon = 10^{-8}\).
The regularization coefficients \(\lambda_1\) and \(\lambda_2\) were grid-searched from \(\{0, 0.2, 0.4, 0.6, 0.8, 1.0\}\).
\subsection{Details of Tr-CMC}
\label{sec:org2b3fe78}
For trace-norm regularized clipped matrix completion, we used the accelerated proximal gradient singular value thresholding algorithm (APG) introduced in \citep{Tohaccelerated2010}. APG is an iterative algorithm in which the gradient of the loss function is used.
Thanks to the differentiability of the squared hinge loss function, we are able to use APG to minimize the CMC objective (Eq.~\eqref{eq:loss-cmc}) with \(\mathcal{R}(\X) = \lambda \|X\|_\trnrm\).
We obtained an implementation of APG for matrix completion from \url{http://www.math.nus.edu.sg/\~mattohkc/NNLS.html}, and modified the code for Tr-CMC.
\subsubsection{Experiments.}
\label{sec:org8e3d584}
In the experiments using real-world data, we used \(T \in \{100, 500\}, \eta = 0.8, L_f = \|\PO\|_\op^2\), and \(\X^{(0)} = \O\).
For the regularization coefficient, the default values proposed by \citet{Tohaccelerated2010} was used,
i.e., using a continuation method to minimize Eq.~\eqref{eq:general-alg-fro} with \(\lambda = \max\{0.7^{t-1}, 10^{-4}\} \|\PO(\M^\c)\|_\op\) for iteration \(t\), to eventually minimize Eq.~\eqref{eq:general-alg-fro} with \(\lambda := 10^{-4}\|\PO(\M^\c)\|_\op\).
\subsection{Details of Fro-CMC}
\label{sec:orga49f619}
This method first parametrizes \(\X\) as \(\mP\Q^\top\), where \(\mP \in \mathbb{R}^{n_1 \times k}, \Q \in \mathbb{R}^{n_2 \times k}\), and minimizes Eq.~\eqref{eq:general-alg-fro} with \(\mathcal{R}(\mP, \Q) = \frac{\lambda}{2}(\|\mP\|_\Fro^2 + \|\Q\|_\Fro^2)\).
Here we use \(\mP^\top = (\p_1 \cdots \p_{n_1}), \Q^\top = (\q_1 \cdots \q_{n_2})\) to denote the (transposed) row vectors of \(\mP\) and \(\Q\).
\subsubsection{Original algorithm.}
\label{sec:orgb727d61}
The minimization objective is not jointly convex in \((\mP, \Q)\). Nevertheless, it is separately convex when one of \(\mP\) and \(\Q\) is fixed.
The idea of alternating least squares (ALS) is to fix \(\mP\) when minimizing Eq.~\eqref{eq:general-alg-fro-orig} with respect to \(\Q\), and vice versa.
In its original form, each update is analytic and takes the form
\begin{equation*}\begin{split}
\q_j &\gets \left(\sum_{i: (i, j) \in \Omega} \p_i \p_i^\top + \lambda \mI\right)^{-1} \left(\sum_{i: (i, j) \in \Omega} M^\clipped_{ij} \p_i\right), \quad j \in [n_2],\\
\p_i &\gets \left(\sum_{j: (i, j) \in \Omega} \q_j \q_j^\top + \lambda \mI\right)^{-1} \left(\sum_{j: (i, j) \in \Omega} M^\clipped_{ij} \q_j\right), \quad i \in [n_1].
\end{split}\end{equation*}
\subsubsection{Proposed algorithm.}
\label{sec:org7e7ad1e}
The squared hinge loss \((x)_+^2\) is differentiable, and its derivative is \((2 x) \Indicator\{x > 0\}\).
Thanks to its differentiability, the loss function Eq.~\eqref{eq:loss-cmc} is also differentiable.
However, in the case of CMC, a closed-form minimizer is not obtained as in the derivation of the original ALS, due to the existence of indicator function in its derivative.
As an alternative, we derive a method to alternately update the parameters by minimizing an approximate objective at each iteration.
Denoting \(z^{(s, t)}_{ij} := \Indicator\{M^\clipped_{ij} > \p_i^{(s)\top} \q_j^{(t)}\}\), we use the following heuristic update rules to approximately obtain the minimizers.
\begin{subequations}\begin{align}
\q_j^{(t)} &\gets \left(\sum_{i: (i, j) \in \Omega \setminus \mathcal{C}} \p_i^{(t-1)} \p_i^{{(t-1)}\top} + \sum_{i: (i, j) \in \mathcal{C}} \p_i^{(t-1)} \p_i^{(t-1)\top} z^{(t-1, t-1)}_{ij} + \lambda \mI\right)^{-1} \notag\\
&\qquad \left(\sum_{i: (i, j) \in \Omega \setminus \mathcal{C}} M^\clipped_{ij} \p_i^{(t-1)} + \sum_{i: (i, j) \in \mathcal{C}} M^\clipped_{ij} z^{(t-1, t-1)}_{ij} \p_i^{(t-1)}\right), \label{eq:als-update-q}\\
\p_i^{(t)} &\gets \left(\sum_{j: (i, j) \in \Omega \setminus \mathcal{C}} \q_j^{(t-1)} \q_j^{(t-1)\top} + \sum_{j: (i, j) \in \mathcal{C}} \q_j^{(t-1)} \q_j^{(t-1)\top} z^{(t-1, t)}_{ij} + \lambda \mI\right)^{-1} \notag\\
&\qquad \left(\sum_{j: (i, j) \in \Omega \setminus \mathcal{C}} M^\clipped_{ij} \q_j^{(t-1)} + \sum_{j: (i, j) \in \mathcal{C}} M^\clipped_{ij} z^{(t-1, t)}_{ij} \q_j^{(t-1)} \right), \label{eq:als-update-p}
\end{align}\end{subequations}
where \(\mI \in \mathbb{R}^{r \times r}\) is the identity matrix.
We iterate between Eq.~\eqref{eq:als-update-q} and Eq.~\eqref{eq:als-update-p} as in Algorithm~\ref{alg:approximate-als}.
For the same reason as DTr-CMC, we let the algorithm start from a matrix whose values are all \(C+1\).
In the algorithm, we let \(\mEP \in \Psp\) and \(\mEQ \in \Qsp\) be the all-one matrices.
\begin{algorithm}
\caption{An approximate alternating least squares algorithm for Fro-CMC.}
\label{alg:approximate-als}
\begin{algorithmic}
\REQUIRE $k \in \mathbb{N}$, $T \in \mathbb{N}$, $\lambda > 0$, and $C \in \mathbb{R}$.
\STATE \textbf{Initialize:} $\mP^{(0)} = \frac{C+1}{\sqrt{k}} \mEP$ and $\Q^{(0)} = \frac{C+1}{\sqrt{k}} \mEQ$
\FOR{$t = 1, \ldots, T$}
\STATE Obtain $\q_j^{(t)}$ according to Eq.~\eqref{eq:als-update-q}
\STATE Obtain $\p_i^{(t)}$ according to Eq.~\eqref{eq:als-update-p}
\ENDFOR
\RETURN $\hM = \mP^{(t)}\Q^{(t)\top}$
\end{algorithmic}
\end{algorithm}
\subsubsection{Experiments.}
\label{sec:orgc176721}
In the experiments using synthetic data, hyperparameters were grid-searched from \(\lambda \in \{0.01, 0.1, 0.5, 1.0\}, k \in \{5l: l \in [8]\}\), and \(T = 200\).
In the experiments using real-world data, hyperparameters were grid-searched from \(\lambda \in \{10^{-1}, 10^{-2}, 10^{-3}\}, k \in \{20 + 4l: l \in [5]\}\), and \(T \in \{500, 1500\}\).
\section{The generation process of the synthetic data \label{app:data-generation}}
\label{sec:org33091f2}
In order to obtain rank-\(r\) matrices with different rates of clipping, synthetic data were generated by the following procedure.
\begin{enumerate}
\item For a fixed \(L \in \mathbb{N}\), we first generated a matrix \(\tilde \M\) whose entries were independent samples from a uniform distribution over \([L]\).
\item We used a non-negative matrix factorization algorithm \citep{LeeAlgorithms2001a} to approximate \(\tilde \M\) with a matrix \(\M\) of rank at most \(r\).
\item We repeated the generation of \(\tilde \M\) until the rank of \(\M\) was exactly \(r\). Note that with this procedure, \(\|\M\|_\infty\) may become larger than \(L\).
\item We randomly split \(\M\) into \(\tilde{\M}^\clipped_\Omega, \M^\valid, \M^\test\) with ratio \((0.8, 0.1, 0.1)\), which were used for training, validation and testing, respectively.
\item We clipped \(\tilde{\M}^\clipped_\Omega\) at the clipping threshold \(C\) to generate \(\McOmega\).
\end{enumerate}

Note that we have also conducted experiments using continuous-valued synthetic data and confirmed the results are similar to discrete-valued cases. The experiments are designed to complement the theoretical findings, and they can be validated regardless of whether the data are discrete.

The visual demonstration of CMC in Figure~\ref{fig:illustration} was generated by the process above with \(r = 2, L = 15, p=1\), and \(C = 10\).
Figure~(\ref{fig:illustration}\subref{fig:res-illust-CMC}) is a result of applying Fro-CMC to the generated matrix.
\begin{table}
\begin{center}
\captionof{table}{\label{tbl:notation}
List of symbols used in the main text.}
\begin{tabular}{|l|l|l|}
\hline
Symbol & Meaning\\
\hline
\(\M\) & Ground truth matrix (rank \(r\))\\
\(\Mc\) & Clipped ground truth matrix\\
\(\McOmega\) & Observed matrix\\
\(\hM\) & Estimated matrix\\
\(\Omega\) & Observed entries\\
\(\C\) & Observed clipped entries\\
\hline
\(\PC\) & The characteristic operator of clipping\\
\(T\) & The information subspace\\
\(\U\mSigma\V^\top\) & Skinny singular value decomposition of \(\M\)\\
\(\UCoherence, \VCoherence\) & The row-wise (column-wise) coherence parameters\\
\hline
\(\mu_0\) & Coherence of \(\M\)\\
\(\jointCoherence\) & Joint coherence of \(\M\)\\
\(\nuC\) & The importance of \(\TrueNonClipped\) for \(T\)\\
\(\rhoFro\) & The information loss due to clipping w.r.t. the Frobenius norm\\
\(\rhoInfty\) & The information loss due to clipping w.r.t. the infinity norm\\
\(\rhoOp\) & The information loss due to clipping w.r.t. the operator norm\\
\(k_0\) & The number of partitions that generated \(\Omega\) (introduced for the theoretical analysis)\\
\(\muG\) & \(\supG \unnCoherence(\X)\)\\
\hline
\(\fMC\) & The loss function for CMC using only the squared loss\\
\(\fCMC\) & The loss function for CMC using both the squared loss and the squared hinge loss\\
\hline
\(\mathbb{R}\) & The set of real numbers\\
\(\mathbb{N}\) & The set of natural numbers\\
\(\Order\) & Landau's asymptotic notation for \(n_1, n_2 \to \infty\)\\
\([n]\) & \(\{1, 2, \ldots, n\}\) where \(n \in \mathbb{N}\)\\
\(X_{ij}\) & Element \((i, j)\) of the matrix \(\X\)\\
\(\Proj_{\mathcal{S}}\) & Projection to \(\mathcal{S}\) (\(\mathcal{S} \subset [n_1] \times [n_2]\));\\
 & the linear operator to set matrix elements outside \(\mathcal{S}\) to zero\\
\(\top\) & The transpose\\
\(\normbar\cdot\normbar\) & The Euclidean norm of vectors\\
\(\normbar\cdot\normbar_\trnrm\) & The trace-norm\\
\(\normbar\cdot\normbar_\op\) & The operator norm\\
\(\normbar\cdot\normbar_\Fro\) & The Frobenius norm\\
\(\normbar\cdot\normbar_\infty\) & The entry-wise infinity norm\\
\(\mathrm{span}(S)\) & The subspace of \(\matsp\) spanned by \(S \subset \matsp\)\\
\(\range\) & The range of a mapping\\
\hline
\end{tabular}
\end{center}
\end{table}
\section{Computation time of the proposed methods \label{app:run-time}}
\label{sec:org2fdbb2d}
Here, we report the computation time of the proposed methods in the experiments in Section~\ref{sec:experiments}.
\subsubsection{Common setup.}
\label{sec:org29d9128}
All the experiments were run on a workstation with Intel(R) Xeon(R) CPU E5-2640 v3 @ 2.60GHz.
The reported running times are the longest wall-clock time for each method among all hyperparameter candidates.
Note that the implementations varied (DTr-CMC and Tr-CMC were in MATLAB and Fro-CMC was in Python) and the levels of code optimization may vary.
\subsubsection{Results.}
\label{sec:orgcc39eb5}
In the experiment using synthetic data sets, the proposed methods ran in 97 (DTr-CMC), 85 (Fro-CMC), and 10 seconds (Tr-CMC).
In the experiment using benchmark data sets, our proposed methods ran on FilmTrust in 706 (DTr-CMC), 606 (Fro-CMC), and 15 seconds (Tr-CMC),
whereas on Movielens 100K, they ran in 324 (DTr-CMC), 609 (Fro-CMC), and 11 seconds (Tr-CMC).
These figures show that our proposed methods are usable for moderately sized matrices.
\subsubsection{Running time on Movielens (20M) dataset.}
\label{sec:org6df0c86}
Here, we also report the running time of our proposed methods on a larger dataset: Movielens (20M)\footnote{\url{http://grouplens.org/datasets/movielens/20m/}}. Movielens (20M) consists of ratings from 138,000 users to 27,000 movies on a scale from \(0.5\) to \(5.0\) with a stride of \(0.5\) (approximately 99.6\% missing).
The running time on Movielens (20M) for our proposed methods were: 11 minutes 51 seconds per iteration (DTr-CMC; only the top 20 singular values were calculated for SVD), 40 minutes 27 seconds (Tr-CMC), and 8 minutes 46 seconds per epoch (Fro-CMC), for the hyperparameter settings that required the longest running times.
\subsubsection{Scaling up the proposed methods.}
\label{sec:orgd9e82f9}
In order to scale up the proposed methods to very large matrices, one can employ existing tricks in combination with our proposed methods, e.g., stochastically approximating subgradients \citep{AvronEfficient2012a}, calculating only the first few singular values, or using stochastic/coordinate gradient descent \citep{MarecekMatrix2017} instead of the ALS-like algorithm.
\section{Proof of Theorem~\ref{thm:cmc-feasible} \label{app:thm1}}
\label{sec:orgca20eb9}
We define \(\oij := \Indicator\{(i, j) \in \Omega\}\) and \(\oijk := \Indicator\{(i, j) \in \Omega_k\}\).
We also define linear operators \(\RO, \RO^\half, \RB\), and \(\ROk\) by \(\RO := \frac{1}{p}\PO, \RO^\half := \frac{1}{\sqrt{p}} \PO, \RB := \frac{1}{p} \PB\), and \(\ROk := \frac{1}{q} \POk\).
Note \(\PA, \PB, \PO, \RO, \RO^\half\) are all self-adjoint.
We denote the identity map by \(\I: \matsp \to \matsp\).
The summations \(\sumij\) indicate the summation over \((i, j) \in [n_1] \times [n_2]\).
The maximum \(\max_{(i, j)}\) indicate the maximum over \((i, j) \in [n_1] \times [n_2]\).
The standard basis of \(\mathbb{R}^{n_1}\) is denoted by \(\{\e_i\}_{i=1}^{n_1}\), and that of \(\mathbb{R}^{n_2}\) by \(\{\f_j\}_{j=1}^{n_2}\).
Even though \(\PC\) is nonlinear, we omit parentheses around its arguments when the order of application is clear from the context (operators are applied from right to left).
For continuous linear operators operating on \(\matsp\), \(\|\cdot\|_\op\) is the operator norm induced by the Frobenius norm.

Theorem~\ref{thm:cmc-feasible} is a simplified statement of the following theorem.
Its proof is based on guarantees of exact recovery for missing entries \citep{CandesExact2009,Rechtsimpler2011a,ChenCompleting2015}, and it is extended to deal with the nonlinearity arising from \(\PC\).
\begin{theorem}[]
Assume \(\rhoFro < \frac{1}{2}, \rhoOp < \frac{1}{4}, \rhoInfty < \frac{1}{2}\), and \(\nuC < \frac{1}{2}\),
and assume the independent and uniform sampling scheme as in Assumption~\ref{assumption:sampling-scheme}.
If for some \(\beta > \max\{\lemOpBetaMin, \lemFroBetaMin, \lemInfBetaMin\}\),
\begin{equation}\label{eq:thm:exact-recovery-guarantee:p}\begin{split}
p \geq \min\left\{1, \max\left\{\frac{1}{n_1 n_2}, \pMinFro, \pMinOpOne,\pMinOpTwo,\pMinInfty,\pMinMain\right\}\right\}
\end{split}\end{equation}
where
\begin{equation*}\begin{split}
\pMinFro &= \lemFroPLowerBound, \\
\pMinOpOne &= \lemOpPLowerBoundForBernsteinCondition, \\
\pMinOpTwo &= \lemOpPLowerBoundForBound, \\
\pMinInfty &= \lemInfPLowerBound, \\
\pMinMain &= \lemMainPLowerBound,
\end{split}\end{equation*}
is satisfied, then the minimizer of Eq.~\eqref{eq:cmc-tr-min} is unique and equal to \(\M\)
with probability at least \(1 - k_0 (\lemFroDelta + \lemInfDelta + \lemOpDelta) - \lemMainDelta\).
\label{thm:exact-recovery-guarantee}
\end{theorem}
\subsection{Proof of Theorem~\ref{thm:cmc-feasible}}
\label{sec:org49d6d75}
We first show how to obtain the simplified Theorem~\ref{thm:cmc-feasible} from Theorem~\ref{thm:exact-recovery-guarantee}.
\begin{proof}
We impute \(\beta = 3\) here. This is justified because, under \(n_1, n_2 \geq 2\), we have \(\lemFroBetaMin \leq 1\) and \(\lemInfBetaMin \leq 2\).

Next we simplify the condition on \(p\).
By denoting \(\alpha = \alphaValue\), we obtain \(k_0 := \kZeroValue \leq \kzeroSubstitution\).
We obtain the condition on \(p\) in Theorem~\ref{thm:cmc-feasible} by the following calculations:
\begin{equation*}\begin{split}
\pMinFro &\leq \mu_0 r \frac{8 \beta}{(1/2 - \rhoFro)^2} \kzeroSubstitution \frac{(n_1 + n_2) \log(n_1 n_2)}{n_1 n_2} \leq \mu_0 r \cSubstitution \orderSubstitution, \\
\pMinOpOne &\leq \frac{8 \beta}{3 (1/4 - \rhoOp)^2} \kzeroSubstitution \frac{\log(n_1 + n_2)}{\max(n_1, n_2)} \leq \cSubstitution \orderSubstitution, \\
\pMinOpTwo &\leq \jointCoherence^2 r \frac{8 \beta}{3(1/4 - \rhoOp)^2} \kzeroSubstitution \frac{\max(n_1, n_2)}{\log(n_1 + n_2)}{n_1 n_2} \leq \jointCoherence^2 r \cSubstitution \orderSubstitution, \\
\pMinInfty &\leq \mu_0 r \frac{8 \beta}{3(1/2 - \rhoInfty)^2} \kzeroSubstitution \frac{(n_1 + n_2) \log(n_1 n_2)}{n_1 n_2} \leq \mu_0 r \cSubstitution \orderSubstitution, \\
\pMinMain &\leq \mu_0 r \frac{8 \beta}{3(1/2 - \nuC)^2} \frac{(n_1 + n_2) \log(n_1 n_2)}{n_1 n_2} \leq \mu_0 r \cSubstitution \orderSubstitution,
\end{split}\end{equation*}
where we used
\begin{equation*}\begin{split}
\frac{\log(n_1 + n_2)}{\max(n_1, n_2)} = \frac{\min(n_1, n_2) \log(n_1 + n_2)}{n_1 n_2} \leq \frac{(n_1 + n_2)\log(n_1 + n_2)}{n_1 n_2} \leq \frac{(n_1 + n_2) \log(n_1 n_2)}{n_1 n_2},
\end{split}\end{equation*}
which follows from \(n_1, n_2 \geq 2\).

Let \(\delta = k_0 (\lemFroDelta + \lemInfDelta + \lemOpDelta) - \lemMainDelta\).
Now we simplify the upper bound on \(\delta\).
\begin{equation*}\begin{split}
\delta &\leq \kzeroSubstitution (e^{1/4} (n_1 n_2)^{-\beta} + 2 (n_1 n_2)^{1 - \beta} + (n_1 + n_2)^{1 - \beta}) - 2(n_1 n_2)^{1-\beta} \\
&= \frac{\alpha + \frac{1}{2 \log 2} \log(n_1 n_2)}{n_1 + n_2} \left(e^{1/4} \frac{n_1 + n_2}{(n_1 n_2)^\beta} + 2 \frac{n_1 + n_2}{(n_1 n_2)^{\beta -1}} + (n_1 + n_2)^{2 - \beta}\right) - 2 (n_1 n_2)^{1-\beta} \\
&= \frac{\alpha + \frac{1}{2 \log 2} \log(n_1 n_2)}{n_1 + n_2} \left(e^{1/4} \left(\frac{n_1 + n_2}{n_1 n_2}\right)^\beta \frac{1}{n_1 + n_2} + 2 \left(\frac{n_1 + n_2}{n_1 n_2}\right)^{\beta - 1} + 1\right)(n_1 + n_2)^{2 -\beta} - 2 (n_1 n_2)^{1-\beta} \\
&= \Order\left(\frac{\log(n_1 n_2)}{n_1 + n_2}\right) (n_1 + n_2)^{2 - \beta}.
\end{split}\end{equation*}
Substituting \(\beta = 3\), we obtain the simplified statement with regard to \(\delta\) in Theorem~\ref{thm:cmc-feasible}.
\end{proof}
\subsection{Preliminary}
\label{sec:org781c129}
Before moving on to the proof, let us note the following property of coherence to be used in the proof.
\begin{property}[]
\begin{equation*}\begin{split}
\|\PT(\Eij)\|_\Fro^2 \leq \frac{n_1 + n_2}{n_1 n_2} \mu_0 r
\end{split}\end{equation*}
\end{property}
\begin{proof}
\begin{equation*}\begin{split}
\|\PT(\Eij)\|_\F^2 &= \|\PU(\Eij)\|_\F^2 + \|\PV(\Eij)\|_\F^2 - \|\PU(\Eij)\|_\F^2 \|\PV(\Eij)\|_\F^2 \\
&\leq \|\PU(\Eij)\|_\F^2 + \|\PV(\Eij)\|_\F^2 \\
&\leq \frac{n_1 + n_2}{n_1 n_2} \mu_0 r.
\end{split}\end{equation*}
\end{proof}

Note that since \((1 - p)^{1/k_0} \leq 1 - (1 / k_0) p\), it follows that \(q \geq (1 / k_0) p\).
We will repeatedly use this relation in proving concentration properties.
\subsection{Main lemma}
\label{sec:org8dd2ea9}
The key element in the main lemma of our proof (Lemma~\ref{lem:feasibility:main}) is to find a matrix in \(\range \POC\) that is approximately a subgradient of \(\|\cdot\|_\trnrm\) at \(\M\).
Such a matrix is called a \emph{dual certificate}.
Its definition is extended to deal with inequality constraints compared to the definitions in previous works \citep{CandesExact2009,Rechtsimpler2011a,ChenCompleting2015}.
\begin{definition}[Dual certificate]
We say that \(\Y \in \matsp\) is a dual certificate if it satisfies
\begin{equation*}\begin{split}
\Y \in \range &\POC, \\
\|\UV - \PT \Y\|_\F &\leq \uvupperbound, \\
\|\PTT \Y\|_\op &< \opnormupperbound. \\
\end{split}\end{equation*}
\label{def:dual-cert}
\end{definition}
By definition of \(\PC\), we have
\begin{equation*}\begin{split}
\langle \PO(\Mc - \M), \Y \rangle \geq 0.
\end{split}\end{equation*}
Given a dual certificate \(\Y\), we can have the following result.
\begin{lemma}[Main lemma]
Assume that a dual certificate \(\Y\) exists and that \(\|\PT \PO\PnC \PT - \PT\PnC\PT\|_\op \leq \frac{1}{2} - \nuC\) holds.
Then, the minimizer of Eq.~\eqref{eq:cmc-tr-min} is unique and is equal to \(\M\).
\label{lem:feasibility:main}
\end{lemma}
\begin{proof}
Note that \(\M\) is in the feasibility set of Eq.~\eqref{eq:cmc-tr-min}.
Let \(\hM \in \matsp\) be another matrix (different from \(\M\)) in the feasibility set and denote \(\H := \hM - \M\).
Since the trace-norm is dual to the operator norm \citep[Proposition 2.1]{RechtGuaranteed2010b}, there exists \(\W \in T^\perp\) which satisfies \(\|\W\|_\op = 1\) and \(\langle \W, \PTT \H \rangle = \|\PTT \H\|_\trnrm\). It is also known that by using this \(\W\), \(\UV + \W\) is a subgradient of \(\|\cdot\|_\trnrm\) at \(\M\) \citep{CandesExact2009}. Therefore, we can calculate
\begin{equation}\label{eq:main:1}\begin{split}
\|\hM\|_\trnrm &= \|\M + \H\|_\trnrm \\
&\geq \|\M\|_\trnrm + \langle \H, \UV + \W \rangle \\
&= \|\M\|_\trnrm + \langle \H, \UV - \PT \Y \rangle + \langle \H, \W - \PTT \Y \rangle + \langle \H, \Y \rangle \\
&\geq \|\M\|_\trnrm + \langle \PT \H, \UV - \PT \Y \rangle + \langle \PTT \H, \W - \PTT \Y \rangle + \langle \H, \Y \rangle,
\end{split}\end{equation}
where we used the self-adjointness of the projection operators, as well as \(\UV \in T\).

From here, we will bound each term in the rightmost equation of Eq.~\eqref{eq:main:1}.
\paragraph{(Lower-bounding $\langle \H, \Y \rangle$ with $0$).}
We have \(\langle \H, \Y \rangle \geq \langle \Mc - \M, \Y \rangle \geq 0\), since
\begin{equation*}\begin{split}
\langle \H, \Y \rangle - \langle \Mc - \M, \Y \rangle = \langle \hM - \Mc, \Y \rangle = \langle \PO(\hM - \Mc), \Y \rangle \geq 0
\end{split}\end{equation*}
can be seen by considering the signs element-wise.

\paragraph{(Lower-bounding $\langle \PTT \H, \PTT (\W - \Y) \rangle$ with $\|\PTT \H\|_\F$).}
We have
\begin{equation*}\begin{split}
\langle \PTT \H, \PTT (\W - \Y) \rangle &= \|\PTT \H\|_\trnrm - \langle \PTT \H, \PTT \Y \rangle \\
&\geq (1 - \|\PTT \Y\|_\op) \|\PTT \H\|_\trnrm \\
&\geq (1 - \|\PTT \Y\|_\op) \|\PTT \H\|_\F.
\end{split}\end{equation*}

\paragraph{(Lower-bounding $\langle \PT \H, \UV - \PT \Y \rangle$ with $\|\PTT \H\|_\F$).}
Now note
\begin{equation*}\begin{split}
\langle \PT \H, \UV - \PT \Y \rangle \geq - \|\PT \H\|_\F \|\UV - \PT \Y\|_\F, \\
\end{split}\end{equation*}
We go on to upper-bound \(\|\PT \H\|_\F\) by \(\|\PTT \H\|_\F\).

Note \(0 = \|\ROh\PnC\H\|_\Fro \geq \|\ROh\PnC\PT \H\|_\Fro - \|\ROh\PnC\PTT\H\|_\Fro\).
Therefore, \(\|\ROh\PnC\PT \H\|_\Fro \geq \|\ROh\PnC\PTT\H\|_\Fro\).
Now
\begin{equation*}\begin{split}
\|\ROh\PnC\PT\H\|_\Fro^2 &= \langle \RO\PnC\PT\H, \PnC\PT\H \rangle \\
&= \langle \RO\PnC\PT\H, \PT\H \rangle \\
&= \|\PT\H\|_\Fro^2 + \langle \PT(\RO \PnC\PT - \PT) \PT\H, \PT\H \rangle \\
&\geq \|\PT\H\|_\Fro^2 - \|\PT\RO\PnC\PT - \PT\|_\op \|\PT\H\|_\Fro^2 \\
&\geq \|\PT\H\|_\Fro^2 - (\|\PT\RO\PnC\PT - \PT\PnC\PT\|_\op + \|\PT\PnC\PT - \PT\|_\op) \|\PT\H\|_\Fro^2 \\
&\geq \left(1 - \left(\frac{1}{2} - \nuC\right) - \nuC\right) \|\PT\H\|_\Fro^2 \\
&= \frac{1}{2} \|\PT\H\|_\Fro^2.
\end{split}\end{equation*}
On the other hand,
\begin{equation*}\begin{split}
\|\ROh \PnC\PTT\H\|_\Fro \leq \frac{1}{\sqrt{p}} \|\PTT\H\|_\Fro.
\end{split}\end{equation*}
Therefore, we have
\begin{equation*}\begin{split}
- \|\PT \H\|_\Fro \geq - \sqrt{\frac{2}{p}} \|\PTT\H\|_\Fro.
\end{split}\end{equation*}

\paragraph{(Finishing the proof).}
Now we are ready to continue the calculation of Eq.~\eqref{eq:main:1} as
\begin{equation*}\begin{split}
\|\hM\|_\trnrm &\geq \|\M\|_\trnrm - \|\UV - \PT \Y\|_\F \|\PT \H\|_\F + (1 - \|\PTT \Y\|_\op) \|\PTT \H\|_\F + 0\\
&\geq \|\M\|_\trnrm - \|\UV - \PT \Y\|_\F \sqrt{\frac{2}{p}} \|\PTT \H\|_\F + (1 - \|\PTT \Y\|_\op) \|\PTT \H\|_\F \\
&\geq \|\M\|_\trnrm + \left(1 - \|\PTT \Y\|_\op - \|\UV - \PT \Y\|_\F \sqrt{\frac{2}{p}} \right) \|\PTT \H\|_\F \\
&> \|\M\|_\trnrm + \left(1 - \frac{1}{2} - \frac{1}{2}\right) \|\PTT \H\|_\F \\
&= \|\M\|_\trnrm.
\end{split}\end{equation*}
Therefore, \(\M\) is the unique minimizer of Eq.~\eqref{eq:cmc-tr-min}.
\end{proof}
From here, we will prove that the conditions of Lemma~\ref{lem:feasibility:main} holds with high probability.
The proof relies on the following three concentration inequalities.
\subsubsection{Concentration inequalities}
\label{sec:orgea5bf65}
\begin{theorem}[Matrix Bernstein inequality {\citep{TroppUserfriendly2012}}]
Let \(\{\Z_k\}_{k=1}^L\) be independent random matrices with dimensions \(d_1 \times d_2\).
If \(\E(\Z_k) = \O\) and \(\|\Z_k\|_\op \leq R\) (a.s.),
then define \(\sigma^2 := \max\left\{\left\|\sumk \E(\Z_k^\top \Z_k)\right\|_\op, \left\|\sumk \E(\Z_k \Z_k^\top)\right\|_\op\right\}\).
Then for all \(t \in \left[0, \frac{\sigma^2}{R}\right]\),
\begin{equation*}\begin{split}
\P\left\{\left\|\sumk \Z_k\right\|_\op \geq t\right\} \leq (d_1 + d_2) \exp \left(\frac{- \frac{3}{8}t^2}{\sigma^2}\right)
\end{split}\end{equation*}
holds.
Therefore, if
\begin{equation}\label{eq:mat-bernstein-cond}\begin{split}
\sqrt{\frac{8}{3}\left(\log \frac{d_1 + d_2}{\delta}\right) \sigma^2} \leq \frac{\sigma^2}{R},
\end{split}\end{equation}
then with probability at least \(1 - \delta\),
\begin{equation*}\begin{split}
\left\|\sumk \Z_k\right\|_\op \leq \sqrt{\frac{8}{3}\left(\log \frac{d_1 + d_2}{\delta}\right) \sigma^2}
\end{split}\end{equation*}
holds.
\label{thm:mat-bernstein}
\end{theorem}
The following theorem is essentially contained in Chapter~6 of \cite{LedouxProbability1991}.
A different version of the following theorem can be found in \citep{GrossRecovering2011}.
\citet{KohlerSubsampled2017} have shown this variant.
\begin{theorem}[Vector Bernstein inequality {\citep{GrossRecovering2011}}]
Let \(\{\bv_k\}_{k=1}^L\) be independent random vectors in \(\mathbb{R}^d\).
Suppose that \(\mathbb{E} \bv_k = \bm{o}\) and \(\|\bv_k\| \leq R\) (a.s.) and put \(\sumk \mathbb{E} \|\bv_k\|^2 \leq \sigma^2\).
Then for all \(t \in \left[0, \frac{\sigma^2}{R}\right]\),
\begin{equation*}\begin{split}
\mathbb{P}\left(\left\|\sumk \bv_k\right\| \geq t\right) \leq \exp\left(- \frac{(t / \sigma -1)^2}{4}\right) \leq \exp \left(- \frac{t^2}{8 \sigma^2} + \frac{1}{4}\right)
\end{split}\end{equation*}
holds.
Therefore, given
\begin{equation}\label{eq:vec-bernstein-cond}\begin{split}
\sigma \sqrt{2 + 8 \log \frac{1}{\delta}} \leq \frac{\sigma^2}{R}
\end{split}\end{equation}
with probability at least \(1 - \delta\),
\begin{equation*}\begin{split}
\left\|\sumk \bv_k\right\| \leq \sigma \sqrt{2 + 8 \log \frac{1}{\delta}}
\end{split}\end{equation*}
holds.
\label{thm:vec-bernstein}
\end{theorem}
\begin{theorem}[Bernstein's inequality for scalars {\citep[Corollary 2.11]{BoucheronConcentration2013}}]
Let \(X_1, \ldots, X_n\) be independent real-valued random variables that satisfy \(|X_i| \leq R\) (a.s.), \(\E[X_i] = 0\), and \(\sum_{i=1}^n \E[X_i^2] \leq \sigma^2\).
Then for all \(t \in \left[0, \frac{\sigma^2}{R}\right]\),
\begin{equation*}\begin{split}
\mathbb{P}\left\{\left|\sum_{i=1}^n X_i\right| \geq t\right\} \leq 2\exp\left(- \frac{\frac{3}{8}t^2}{\sigma^2}\right).
\end{split}\end{equation*}
holds.
Therefore, given
\begin{equation}\label{eq:scalar-bernstein-cond}\begin{split}
\sqrt{\frac{8}{3} \sigma^2 \log \frac{2}{\delta}} \leq \frac{\sigma^2}{R},
\end{split}\end{equation}
with probability at least \(1 - \delta\),
\begin{equation*}\begin{split}
\left|\sum_{i=1}^n X_i\right| \leq \sqrt{\frac{8}{3} \sigma^2 \log \frac{2}{\delta}}.
\end{split}\end{equation*}
holds.
\label{thm:scalar-bernstein}
\end{theorem}
\subsection{Condition for \(\|\PT \PO\PnC \PT - \PT\PnC\PT\|_\op\) in Lemma~\ref{lem:feasibility:main} to hold with high probability}
\label{sec:org2463d98}
\begin{lemma}[]
If for some \(\beta > \lemMainBetaMin\),
\begin{equation*}\begin{split}
p \geq \min\left\{1, \pMinMain\right\}
\end{split}\end{equation*}
is satisfied, then
\begin{equation}\label{eq:proof-exact-recovery-3}\begin{split}
\|\PT\RO\PnC\PT - \PT\PnC\PT\|_\op \leq \frac{1}{2} - \nuC
\end{split}\end{equation}
holds with probability at least \(1 - \lemMainDelta\).
\label{lem:main-lemma-conc}
\end{lemma}
\begin{proof}
If \(p = 1\), then Eq.~\eqref{eq:proof-exact-recovery-3} holds.
From here, we assume \(1 \geq p \geq \pMinMain\).

Any matrix \(\Z\) can be decomposed into a sum of elements \(\Z = \sum_{(i, j)} \langle \Z, \Eij \rangle \Eij\) and therefore, \(\PT(\Z) = \sum_{(i, j)} \langle \PT(\Z), \Eij \rangle \Eij = \sum_{(i, j)} \langle \Z, \PT (\Eij) \rangle \Eij\).
Thus, by letting \(\bcij := \Indicator\{(i, j) \in \TrueNonClipped\}\),
\begin{equation}\begin{split}
(\PT\RO\PnC\PT - \PT\PnC\PT)(\Z) &= (\PT \RO\PnC - \PT\PnC) \left(\sum_{(i, j)} \langle \PT \Z, \Eij \rangle \Eij\right) \\
&= (\PT \RO - \PT) \left(\sum_{(i, j)} \langle \PT \Z, \Eij \rangle \bcij\Eij\right) \\
&= \sum_{(i, j)} \left(\frac{\omega_{ij}}{p} - 1\right) \bcij \langle \Z, \PT (\Eij) \rangle \PT (\Eij) \\
&=: \sum_{(i, j)} S_{ij}(\Z).
\end{split}\end{equation}

Now it is easy to verify that \(\E\left[S_{ij}\right] = \mO\).
We also have
\begin{equation*}\begin{split}
\|S_{ij}(\Z)\|_\F &\leq \frac{1}{p} \|\PT (\Eij)\|_\F^2 \|\Z\|_\F \leq \frac{1}{p} \frac{n_1 + n_2}{n_1 n_2}\mu_0 r \|\Z\|_\F \\
\end{split}\end{equation*}
Therefore, \(\|S_{ij}\|_\op \leq \frac{n_1 + n_2}{qn_1n_2} \mu_0 r\).
For any matrix \(\Z \in \matsp\),
\begin{equation}\begin{split}
\left\|\sumij \E[S_{ij}^2(\Z)]\right\|_\F &= \left\|\sumij \E\left[\left(\frac{\oij}{p} - 1\right) \bcij \left\langle \left(\frac{\oij}{p} - 1\right) \bcij \langle \Z, \PT(\Eij) \rangle \PT(\Eij), \PT(\Eij) \right\rangle \PT(\Eij)\right]\right\|_\F \\
&=\left\|\sumij \E\left[\left(\frac{\oij}{p} - 1\right)^2 \bcij \langle \Z, \PT(\Eij) \rangle \|\PT (\Eij)\|_\F^2 \PT(\Eij)\right]\right\|_\F \\
&=\left\|\sumij \E\left[\left(\frac{\oij}{p} - 1\right)^2\right] \bcij \|\PT (\Eij)\|_\F^2 \langle \PT \Z, \Eij \rangle \PT(\Eij)\right\|_\F \\
&=\left\|\sumij \frac{1 - p}{p} \|\PT (\Eij)\|_\F^2 \langle \PT \Z, \Eij \rangle \PT(\bcij \Eij)\right\|_\F \\
&\leq \left(\max_{ij} \frac{1 - p}{p} \|\PT (\Eij)\|_\F^2\right) \left\|\sumij \langle \PT(\Z), \Eij \rangle \PT(\bcij \Eij)\right\|_\F \\
&\leq \left(\frac{n_1 + n_2}{p n_1 n_2}\mu_0 r\right)\left\|\PT \PnC \PT(\Z)\right\|_\F \\
&\leq \left(\frac{n_1 + n_2}{p n_1 n_2}\mu_0 r\right)\|\Z\|_\F,
\end{split}\end{equation}
where we used \(S^\top_{ij} S_{ij} = S_{ij} S_{ij}^\top = S_{ij}^2\) (one can see this by checking \(\langle \e_k \f_l^\top, S_{ij}(\e_m \f_n^\top) \rangle = \langle S_{ij}(\e_k \f_l^\top), \e_m \f_n^\top \rangle\)).
Therefore, \(\|\sumij \E S_{ij}^2\|_\op \leq \frac{n_1 + n_2}{p n_1 n_2} \mu_0 r\).

Let \(R := \frac{n_1 + n_2}{p n_1 n_2} \mu_0 r, \sigma^2 := \frac{n_1 + n_2}{p n_1 n_2} \mu_0 r\), and \(\delta = \lemMainDelta\).
Under the condition that
\begin{equation*}\begin{split}
p \geq \pMinMain = \lemMainPLowerBound,
\end{split}\end{equation*}
the condition~\eqref{eq:mat-bernstein-cond} of Theorem~\ref{thm:mat-bernstein} is satisfied, because
\begin{equation}\begin{split}
\sqrt{\frac{8}{3} \sigma^2 \log\left(\frac{2n_1n_2}{\delta}\right)} = \sqrt{\frac{8}{3} \beta \mu_0 r \frac{n_1 + n_2}{p n_1 n_2} \log(n_1 n_2)} \leq \frac{1}{2} - \nuC \leq 1 = \frac{\sigma^2}{R}.
\end{split}\end{equation}
Therefore, applying Theorem~\ref{thm:mat-bernstein} with \((d_1, d_2) = (n_1 n_2, n_1 n_2)\), we obtain
\begin{equation*}\begin{split}
\|\PT\RO\PT - \PT\|_\op \leq \sqrt{\frac{8}{3}\beta \mu_0 r \frac{n_1 + n_2}{p n_1 n_2} \log(n_1 n_2)} \leq \frac{1}{2} - \nuC
\end{split}\end{equation*}
holds with probability at least \(1 - \lemMainDelta\).
\end{proof}
\subsection{A dual certificate exists with high probability}
\label{sec:org6bdccc2}
Here we show how to construct a dual certificate \(\Y \in \matsp\).
It is based on the golfing scheme, which is a proof technique that has been used in constructing dual certificates in conventional settings of MC \citep{GrossRecovering2011,CandesRobust2011,ChenCompleting2015}.
However, for the problem of CMC, the golfing scheme needs to include the characteristic operator in its definition.
\begin{definition}[Generalized golfing scheme]
We recursively define \(\{\W_k\}_{k=0}^{k_0}\) by
\begin{equation*}\begin{split}
\begin{cases}
\W_0 :&= \mO \\
\Dk :&= \UV - \Wk \\
\Wk :&= \Wkminus + \ROkC \PT \Dkminus = \UV - (\I - \ROkC \PT) \Dkminus \\
\end{cases}
\end{split}\end{equation*}
where \(\ROkC(\cdot) := \ROk(\PC(\cdot))\),
and define \(\Y := \W_{k_0}\).
\label{def:generalized-golfing-scheme}
\end{definition}
Due to the non-linearity of \(\PC\), the fact that \(\Y\) is a dual certificate cannot be established in the same way as in existing proofs.
We first establish a lemma to claim the existence of a dual certificate under deterministic conditions,
and then provide concentration properties to prove that the conditions hold with high probability under certain conditions on \(\PC\).
\begin{lemma}[Existence of a dual certificate]
If for some \(\beta > \max\{\lemOpBetaMin, \lemFroBetaMin, \lemInfBetaMin\}\),
\begin{equation}\label{eq:lem:exist-dual-cert-cond}\begin{split}
p \geq \min\left\{1, \max\left\{\frac{1}{n_1 n_2}, \pMinFro, \pMinOpOne, \pMinOpTwo, \pMinInfty \right\} \right\}
\end{split}\end{equation}
is satisfied, then
the matrix \(\Y \in \matsp\) defined by Def.~\ref{def:generalized-golfing-scheme} is a dual certificate (Def.~\ref{def:dual-cert}) with probability at least \(1 - k_0 (\lemFroDelta + \lemInfDelta + \lemOpDelta)\).
\label{lem:existence-dual-cert}
\end{lemma}
\begin{proof}
By construction, we have \(\Y \in \range \POC\).
From here, we show the other two conditions of the dual certificate. In the proof, we will use Prop.~\ref{prop:rho-zero-meaning} below.

\paragraph{(Upper bounding $\|\UV - \PT \Y\|_\Fro$).}

We confirm by recursion that if Eq.~\eqref{eq:lem-fro-conc-result} holds for all \(k \in [k_0]\), then we have \(\|\PT \Dk\|_\Fro \leq \|\UV\|_\Fro\). First, we have \(\|\PT \Delta_0\|_\Fro = \|\UV\|_\Fro\). Second, if \(\|\PT \Dkminus\|_\Fro \leq \|\UV\|_\F\), then
\begin{equation*}\begin{split}
\|\PT \Dk\|_\F &= \|\PT (\UV - \Wk)\|_\F \\
&= \|\UV - \PT \Wkminus - \PT \ROkC \PT \Dkminus\|_\F \\
&\leq \|\PT \Dkminus - \PT\PC\PT \Dkminus\|_\F + \|\PT \PC \PT \Dkminus - \PT\ROkC\PT \Dkminus\|_\F \\
&\leq \rhoFro \|\PT \Dkminus\|_\F + \left(\frac{1}{2} - \rhoFro\right) \|\PT \Dkminus\|_\F \\
&= \frac{1}{2}\|\PT \Dkminus\|_\F \leq \|\UV\|_\F
\end{split}\end{equation*}
Now, by the same recursion formula, we can show \(\|\PT \Delta_{k_0}\|_\F \leq \left(\frac{1}{2}\right)^{k_0}\|\PT \Delta_0\|_\Fro\).
Therefore, under the condition Eq.~\eqref{eq:lem:exist-dual-cert-cond}, by the union bound,
we have Eq.~\eqref{eq:lem-fro-conc-result} for all \(k \in [k_0]\) with probability at least \(1 - k_0 \lemFroDelta\) and
\begin{equation*}\begin{split}
\|\UV - \PT \Y\|_\F = \|\PT \Delta_{k_0}\|_\F &\leq \left(\frac{1}{2}\right)^{k_0} \|\PT \Delta_0\|_\F \\
&\leq \halfPowerKZeroValue\|\UV\|_\F \\
&\leq \uvupperboundTimesSqrtR \|\UV\|_\F \\
&= \uvupperboundTimesSqrtR \sqrt{r} \\
&= \uvupperbound.
\end{split}\end{equation*}
because \(k_0 = \kZeroValue\), where we used \(\frac{1}{n_1n_2} \leq p\).

\paragraph{(Upper bounding $\|\PTT \Y\|_\op$).}

By a similar argument of recursion as above with Eq.~\eqref{eq:lem-inf-conc-result} in Lemma~\ref{lem:inf-conc},
we can prove that for all \(k \in [k_0]\), \(\|\PT \Dk\|_\infty \leq \|\UV\|_\infty\) and \(\|\PT \Dk\|_\infty \leq \frac{1}{2}\|\PT \Dkminus\|_\infty\), with probability at least \(1 - k_0 \lemInfDelta\) under the condition Eq.~\eqref{eq:lem:exist-dual-cert-cond}.
Similarly, with Eq.~\ref{eq:lem-op-conc-result} in Lemma~\ref{lem:op-conc} and using Prop.~\ref{prop:rho-zero-meaning}, we obtain for all \(k \in [k_0]\), \(\|(\ROkC - \I) (\PT \Dkminus)\|_\op \leq \frac{1}{4 \|\UV\|_\infty} \|\PT \Dkminus\|_\infty\), with probability at least \(1 - k_0 \lemOpDelta\) under the condition Eq.~\eqref{eq:lem:exist-dual-cert-cond}.
Therefore, under the condition Eq.~\eqref{eq:lem:exist-dual-cert-cond}, with probability at least \(1 - k_0 (\lemInfDelta + \lemOpDelta)\), we have
\begin{equation*}\begin{split}
\|\PTT Y\|_\op &= \left\|\PTT \sum_{k=1}^{k_0} \ROkC \PT (\Dkminus)\right\|_\op \\
&\leq \sum_{k=1}^{k_0} \|\PTT \ROkC \PT(\Dkminus)\|_\op \\
&= \sum_{k=1}^{k_0} \|\PTT \ROkC \PT(\Dkminus) - \PTT \PT \Dkminus\|_\op \\
&\leq \sum_{k=1}^{k_0} \|(\ROkC - \I) (\PT \Dkminus)\|_\op \\
&\leq \sum_{k=1}^{k_0} \frac{1}{4 \|\UV\|_\infty}  \|\PT \Dkminus\|_\infty \\
&\leq \sum_{k=1}^{k_0} 2^{-k + 1} \frac{1}{4 \|\UV\|_\infty} \|\PT \Delta_0\|_\infty \\
&< \frac{1}{2}.
\end{split}\end{equation*}
By taking the union bound, we have the lemma.
\end{proof}

In the recursion formula, we have used the following property yielding from the definition of \(\rhoOp\) (Def.~\ref{def:p-star-lip}).
\begin{property}[]
\begin{equation*}\begin{split}
\rhoOp \geq \|\U\V^\top\|_\infty \left(\sup_{\Z \in T \setminus \{\mO\}: \|\Z\|_\infty \leq \|\U\V^\top\|_\infty} \frac{\|\PC \Z - \Z\|_\op}{\|\Z\|_\infty}\right)
\end{split}\end{equation*}
\label{prop:rho-zero-meaning}
\end{property}
\begin{proof}
We have \(\{\Z \in T: \|\Z\|_\infty \leq \|\UV\|_\infty\} \subset \{\Z \in T: \|\Z\|_\op \leq \sqrt{n_1 n_2} \|\UV\|_\op\}\),
because we can obtain \(\|\Z\|_\op \leq \sqrt{n_1 n_2} \|\Z\|_\infty \leq \sqrt{n_1 n_2}\|\UV\|_\infty \leq \sqrt{n_1 n_2}\|\UV\|_\op\) from \(\|\Z\|_\infty \leq \|\UV\|_\infty\).
Here, we used \(\|\Z\|_\op \leq \sqrt{n_1 n_2} \|\Z\|_\infty\) and \(\|\Z\|_\infty \leq \|\Z\|_\op\).
Therefore,
\begin{equation*}\begin{split}
\rhoOp &= \sqrt{r} \jointCoherence \left(\sup_{\Z \in T \setminus \{\mO\}: \|\Z\|_\op \leq \sqrt{n_1 n_2} \|\UV\|_\op} \frac{\|\PC \Z - \Z\|_\op}{\|\Z\|_\op}\right) \\
&=\sqrt{n_1 n_2}\|\UV\|_\infty \left(\sup_{\Z \in T \setminus \{\mO\}: \|\Z\|_\op \leq \sqrt{n_1 n_2} \|\UV\|_\op} \frac{\|\PC \Z - \Z\|_\op}{\|\Z\|_\op}\right) \\
&\geq \|\UV\|_\infty \left(\sup_{\Z \in T \setminus \{\mO\}: \|\Z\|_\infty \leq \|\UV\|_\infty} \frac{\|\PC \Z - \Z\|_\op}{\frac{1}{\sqrt{n_1 n_2}}\|\Z\|_\op}\right) \\
&\geq \|\UV\|_\infty \left(\sup_{\Z \in T \setminus \{\mO\}: \|\Z\|_\infty \leq \|\UV\|_\infty} \frac{\|\PC \Z - \Z\|_\op}{\|\Z\|_\infty}\right). \\
\end{split}\end{equation*}
\end{proof}
\subsubsection{Concentration properties}
\label{sec:org5bf9c47}
From here, we denote \(\PCij(\cdot) := (\PC(\cdot))_{ij}\).
\begin{lemma}[Frobenius norm concentration]
Assume that \(\rhoFro < \frac{1}{2}\), and that for some \(\beta > \lemFroBetaMin\),
\begin{equation*}\begin{split}
p \geq \min\left\{1, \pMinFro\right\},
\end{split}\end{equation*}
is satisfied.
Let \(k \in \{1, \ldots, k_0\}\).
Then, given \(\PT \Dkminus\) that is independent of \(\Omega_k\),
we have
\begin{equation}\label{eq:lem-fro-conc-result}\begin{split}
\|\PT\PC\PT \Dkminus - \PT\ROkC\PT \Dkminus\|_\F \leq \left(\frac{1}{2} - \rhoFro\right) \|\PT \Dkminus\|_\F
\end{split}\end{equation}
with probability at least \(1 - \lemFroDelta\).
\label{lem:fro-conc}
\end{lemma}
\begin{proof}
If \(p = 1\), then we have \(q = 1\), therefore Eq.~\eqref{eq:lem-fro-conc-result} holds.
Thus, from here, we assume \(1 \geq p \geq \pMinFro\).

First note that \(\|\PT\PC\PT \Dkminus - \PT\ROkC\PT \Dkminus\|_\F\) can be decomposed as
\begin{equation*}\begin{split}
&\|\PT\PC\PT \Dkminus - \PT\ROkC\PT \Dkminus\|_\F \\
&=\left\|\PT \sumij \left(1 - \frac{\oijk}{q}\right) \PCij(\langle \Eij, \PT \Dkminus \rangle)\Eij\right\|_\F \\
&=\left\|\sumij \left(1 - \frac{\oijk}{q}\right) \PCij(\langle \Eij, \PT \Dkminus \rangle)\PT(\Eij)\right\|_\F \\
&=: \left\|\sumij \S_{ij}\right\|_\F.
\end{split}\end{equation*}
From here, we check the conditions for the vector Bernstein inequality (Theorem~\ref{thm:vec-bernstein}).
Now it is easy to verify that \(\E[\S_{ij}] = \O\).
We also have
\begin{equation*}\begin{split}
\|\S_{ij}\|_\F &= \left(1 - \frac{\oijk}{q}\right)|\PCij(\langle \Eij, \PT \Dkminus\rangle)| \|\PT (\Eij)\|_\F\\
&\leq \frac{1}{q} |\langle \Eij, \PT \Dkminus\rangle| \|\PT (\Eij)\|_\F \\
&\leq \frac{1}{q}\|\PT (\Eij)\|_\F^2 \|\PT \Dkminus\|_\F \\
&\leq \frac{1}{q} \frac{n_1 + n_2}{n_1 n_2} \mu_0 r \|\PT \Dkminus\|_\F.
\end{split}\end{equation*}
On the other hand,
\begin{equation*}\begin{split}
\sumij \E \|\S_{ij}\|_\F^2 &= \sumij \E\left[\left(1 - \frac{\oijk}{q}\right)^2\right] \PCij(\langle \Eij, \PT \Dkminus \rangle)^2 \|\PT (\Eij)\|_\F^2 \\
&= \frac{1 - q}{q} \sumij \PCij(\langle \Eij, \PT \Dkminus \rangle)^2 \|\PT (\Eij)\|_\F^2 \\
&\leq \frac{1 - q}{q} \sumij \langle \Eij, \PT \Dkminus \rangle^2 \|\PT (\Eij)\|_\F^2 \\
&\leq \frac{1 - q}{q} \maxij \left\{\|\PT (\Eij)\|_\F^2\right\} \sumij \langle \Eij, \PT \Dkminus \rangle^2 \\
&= \frac{1 - q}{q} \maxij \|\PT (\Eij)\|_\F^2 \|\PT \Dkminus\|_\F^2 \\
&\leq \frac{1}{q} \maxij \|\PT (\Eij)\|_\F^2 \|\PT \Dkminus \|_\F^2\\
&= \frac{1}{q} \left(\frac{n_1 + n_2}{n_1 n_2} \mu_0 r\right) \|\PT \Dkminus \|_\F^2. \\
\end{split}\end{equation*}

Let \(R := \frac{n_1 + n_2}{q n_1 n_2} \mu_0 r \|\PT \Dkminus\|_\F, \sigma^2 := \frac{n_1 + n_2}{q n_1 n_2} \mu_0 r \|\PT \Dkminus\|_\F^2\), and \(\delta = \lemFroDelta\).
Under the condition
\begin{equation*}\begin{split}
q \geq \frac{p}{k_0} \geq \frac{\pMinFro}{k_0} = \lemFroQLowerBound,
\end{split}\end{equation*}
the condition~\eqref{eq:vec-bernstein-cond} of Theorem~\ref{thm:vec-bernstein} is satisfied, because
\begin{equation*}\begin{split}
\sqrt{\left(2 + 8 \log \frac{1}{\delta}\right) \sigma^2} = \sqrt{8 \beta \log(n_1 n_2) \frac{n_1 + n_2}{q n_1 n_2} \mu_0 r} \|\PT \Dkminus\|_\F \leq \left(\frac{1}{2} - \rhoFro\right) \|\PT \Dkminus\|_\F \leq \|\PT \Dkminus\|_\F = \frac{\sigma^2}{R}.
\end{split}\end{equation*}
Therefore, applying Theorem~\ref{thm:vec-bernstein} with \(d = n_1 n_2\), we obtain
\begin{equation*}\begin{split}
\left\|\sumij \S_{ij}\right\|_\F \leq \sqrt{\left(2 + 8 \log \frac{1}{\delta}\right) \sigma^2} \leq \left(\frac{1}{2} - \rhoFro\right) \|\PT \Dkminus\|_\F
\end{split}\end{equation*}
with probability at least \(1 - \lemFroDelta\).
\end{proof}
\begin{lemma}[Operator norm concentration]
Assume that \(\rhoOp < \frac{1}{4}\), and that for some \(\beta > \lemOpBetaMin\),
\begin{equation*}\begin{split}
p \geq \min\left\{1, \max\left\{\pMinOpOne, \pMinOpTwo\right\}\right\}.
\end{split}\end{equation*}
is satisfied.
Let \(k \in \{1, \ldots, k_0\}\).
Then, given \(\PT \Dkminus\) that is independent of \(\Omega_k\),
we have
\begin{equation}\label{eq:lem-op-conc-result}\begin{split}
\|(\ROkC - \PC) (\PT \Dkminus)\|_\op \leq \left(\frac{1}{4} - \rhoOp\right) \frac{1}{\|\UV\|_\infty} \|\PT \Dkminus\|_\infty
\end{split}\end{equation}
with probability at least \(1 - \lemOpDelta\).
\label{lem:op-conc}
\end{lemma}
\begin{proof}
If \(p = 1\), then we have \(q = 1\), therefore Eq.~\eqref{eq:lem-op-conc-result} holds.
Thus, from here, we assume \(1 \geq p \geq \max\{\pMinOpOne, \pMinOpTwo\}\).

First note that \((\ROkC - \PC) (\PT \Dkminus)\) can be decomposed as
\begin{equation*}\begin{split}
\|(\ROkC - \PC) (\PT \Dkminus)\|_\op &= \left\|\sumij \left(\frac{\oijk}{q} - 1\right) \PCij(\langle \Eij, \PT \Dkminus \rangle)  \Eij\right\|_\op \\
&=: \left\|\sumij \S_{ij}\right\|_\op.
\end{split}\end{equation*}
From here, we check the conditions for the matrix Bernstein inequality (Theorem~\ref{thm:mat-bernstein}).

Now it is easy to verify that \(\E [\S_{ij}] = \O\).
We also have
\begin{equation*}\begin{split}
\|\S_{ij}\|_\op &\leq \frac{1}{q} |\PCij(\langle \Eij, \PT \Dkminus \rangle)| \|\Eij\|_\op \\
&\leq \frac{1}{q} |\langle \Eij, \PT \Dkminus \rangle| \cdot 1 \\
&\leq \frac{1}{q} \|\PT \Dkminus\|_\infty.
\end{split}\end{equation*}
On the other hand,
\begin{equation*}\begin{split}
\left\|\E\left[\sumij \S_{ij}^\top \S_{ij}\right]\right\|_\op &= \left\|\sumij \PCij(\langle \Eij, \PT \Dkminus \rangle)^2 \E \left[\left(\frac{\oijk}{q} - 1\right)^2\right] \f_j \e_i^\top \Eij\right\|_\op \\
&= \frac{1 - q}{q} \left\|\sumij \PCij(\langle \Eij, \PT \Dkminus \rangle)^2 \f_j \f_j^\top\right\|_\op \\
&= \frac{1 - q}{q} \max_j \sum_i \PCij(\langle \Eij, \PT \Dkminus \rangle)^2 \\
&\leq \frac{1 - q}{q} n_1 \maxij \PCij(\langle \Eij, \PT \Dkminus \rangle)^2 \\
&\leq \frac{1 - q}{q} n_1 \maxij \langle \Eij, \PT \Dkminus \rangle^2 \\
&\leq \frac{1}{q} \max(n_1, n_2) \|\PT \Dkminus\|_\infty^2, \\
\left\|\E\left[\sumij \S_{ij} \S_{ij}^\top\right]\right\|_\op &\leq \frac{1}{q} \max(n_1, n_2) \|\PT \Dkminus\|_\infty^2,
\end{split}\end{equation*}
where we used the fact that \(\sumij \PCij(\langle \Eij, \PT \Dkminus \rangle)^2 \f_j \f_j^\top\) is a diagonal matrix whose \((j,j)\)-th element equals \(\sum_i \PCij(\langle \Eij, \PT \Dkminus \rangle)^2\)
and that the operator norm, in the case of a diagonal matrix, is equal to the absolute value of the maximum diagonal element.

Let \(R := \frac{1}{q} \|\PT \Dkminus\|_\infty, \sigma^2 := \frac{1}{q} \max(n_1, n_2) \|\PT \Dkminus\|_\infty^2\), and \(\delta = \lemOpDelta\).
Under the condition
\begin{equation*}\begin{split}
q \geq \frac{p}{k_0} \geq \frac{\pMinOpOne}{k_0} = \lemOpQLowerBoundForBernsteinCondition,
\end{split}\end{equation*}
the condition Eq.~\eqref{eq:mat-bernstein-cond} of Theorem~\ref{thm:mat-bernstein} is satisfied because
\begin{equation*}\begin{split}
\sqrt{\frac{8}{3} \log\left(\frac{n_1 + n_2}{\delta}\right)\sigma^2} &= \lemOpBoundCoeff \|\PT \Dkminus\|_\infty \\
&\leq \left(\frac{1}{4} - \rhoOp\right)\max(n_1, n_2) \|\PT \Dkminus\|_\infty \leq \max(n_1, n_2) \|\PT \Dkminus\|_\infty = \frac{\sigma^2}{R}.
\end{split}\end{equation*}
Therefore, applying Theorem~\ref{thm:mat-bernstein} with \((d_1, d_2) = (n_1, n_2)\) and noting
\begin{equation*}\begin{split}
q \geq \frac{p}{k_0} \geq \frac{\pMinOpTwo}{k_0} = \lemOpQLowerBoundForBound,
\end{split}\end{equation*}
we obtain
\begin{equation*}\begin{split}
\left\|\sumij \S_{ij}\right\|_\op &\leq \sqrt{\frac{8}{3} \log\left(\frac{n_1 + n_2}{\delta}\right)\sigma^2} \\
&= \lemOpBoundCoeff \|\PT \Dkminus\|_\infty \leq \left(\frac{1}{4} - \rhoOp\right) \frac{1}{\|\UV\|_\infty} \|\PT \Dkminus\|_\infty
\end{split}\end{equation*}
with probability at least \(1 - \lemOpDelta\).
\end{proof}

\begin{lemma}[Infinity norm concentration]
Assume that \(\rhoInfty < \frac{1}{2}\), and that for some \(\beta > \lemInfBetaMin\),
\begin{equation*}\begin{split}
p \geq \min\left\{1, \pMinInfty\right\},
\end{split}\end{equation*}
is satisfied.
Let \(k \in \{1, \ldots, k_0\}\).
Then, given \(\PT \Dkminus\) that is independent of \(\Omega_k\),
we have
\begin{equation}\label{eq:lem-inf-conc-result}\begin{split}
\|(\PT \ROkC \PT - \PT \PC \PT)(\PT \Dkminus)\|_\infty \leq \left(\frac{1}{2} - \rhoInfty\right) \|\PT \Dkminus\|_\infty
\end{split}\end{equation}
with probability at least \(1 - \lemInfDelta\).
\label{lem:inf-conc}
\end{lemma}
\begin{proof}
If \(p = 1\), then we have \(q = 1\), therefore Eq.~\eqref{eq:lem-inf-conc-result} holds.
Thus, from here, we assume \(1 \geq p \geq \pMinInfty\).

First note that \((\PT \PC \PT - \PT \ROkC \PT)(\PT \Dkminus)\) can be decomposed as
\begin{equation*}\begin{split}
\|(\PT \PC \PT - \PT \ROkC \PT)(\PT \Dkminus)\|_\infty &= \left\|\PT \sumij\left(1 - \frac{\oijk}{q}\right)\PCij(\langle \Eij, \PT \Dkminus \rangle) \Eij\right\|_\infty \\
&= \left\|\sumij \S_{ij}\right\|_\infty
\end{split}\end{equation*}
Therefore we investigate the elements of \(\S_{ij}\). The \((a, b)\)-th element of \(\S_{ij}\) is \(s_{ab}^{ij} := \langle \Eab, \S_{ij} \rangle\), and
From here, we check the conditions for the scalar Bernstein inequality (Theorem~\ref{thm:scalar-bernstein}).
It is easy to verify that \(\E s^{ij}_{ab} = 0\).
We also have
\begin{equation*}\begin{split}
|s^{ij}_{ab}| &= \left|\left(\frac{\oijk}{q} - 1\right)\PCij(\langle \Eij, \PT \Dkminus \rangle) \langle \PT (\Eij), \Eab\rangle\right| \\
&\leq \frac{1}{q} \left|\langle \Eij, \PT \Dkminus \rangle| \cdot |\langle \PT (\Eij), \PT (\Eab)\rangle\right| \\
&\leq \frac{1}{q} \|\PT \Dkminus\|_\infty \|\PT (\Eij)\|_\Fro \|\PT (\Eab)\|_\Fro \\
&\leq \frac{1}{q} \mu_0 r \frac{n_1 + n_2}{n_1 n_2}\|\PT \Dkminus\|_\infty.
\end{split}\end{equation*}
On the other hand,
\begin{equation*}\begin{split}
\E \left[\sumij (s^{ij}_{ab})^2\right] &= \E \left[ \sumij \langle \Eab, \left(\frac{\oijk}{q} - 1\right) \PCij(\langle \Eij, \PT \Dkminus \rangle) \PT (\Eij)\rangle^2\right] \\
&= \sumij \E \left[\left(\frac{\oijk}{q} - 1\right)^2 \PCij(\langle \Eij, \PT \Dkminus \rangle)^2 \langle \Eab, \PT (\Eij)\rangle^2\right] \\
&\leq \sumij \E \left[\left(\frac{\oijk}{q} - 1\right)^2\right] \langle \Eij, \PT \Dkminus \rangle^2 \langle \Eab, \PT (\Eij)\rangle^2 \\
&= \frac{1 - q}{q} \sumij \langle \Eij, \PT \Dkminus \rangle^2 \langle \Eab, \PT (\Eij)\rangle^2 \\
&\leq \frac{1 - q}{q} \|\PT \Dkminus\|_\infty^2 \sumij \langle \PT (\Eab), \Eij\rangle^2 \\
&= \frac{1 - q}{q} \|\PT \Dkminus\|_\infty^2 \|\PT (\Eab)\|_\Fro^2 \\
&\leq \frac{1}{q} \mu_0 r \frac{n_1 + n_2}{n_1 n_2} \|\PT \Dkminus\|_\infty^2.
\end{split}\end{equation*}

Let \(R := \mu_0 r \frac{n_1 + n_2}{q n_1 n_2} \|\PT \Dkminus\|_\infty, \sigma^2 := \mu_0 r \frac{n_1 + n_2}{q n_1 n_2} \|\PT \Dkminus\|_\infty^2\), and \(\delta = \lemInfSingleDelta\).
Under the condition that
\begin{equation*}\begin{split}
q \geq \frac{p}{k_0} \geq \frac{\pMinInfty}{k_0} = \lemInfQLowerBound,
\end{split}\end{equation*}
the condition Eq.~\eqref{eq:scalar-bernstein-cond} of Theorem~\ref{thm:scalar-bernstein} is satisfied, because
\begin{equation*}\begin{split}
\sqrt{\frac{8}{3} \log \frac{2}{\delta} \sigma^2} = \sqrt{\frac{8}{3}\beta \log(n_1 n_2)} \leq \sqrt{\frac{8}{3} \beta \log(n_1 n_2) \mu_0 r \frac{n_1 + n_2}{q n_1 n_2}} \|\PT \Dkminus\|_\infty \leq \left(\frac{1}{2} - \rhoInfty\right) \|\PT \Dkminus\|_\infty = \frac{\sigma^2}{R},
\end{split}\end{equation*}
Therefore, applying Theorem~\ref{thm:scalar-bernstein} with \(n = n_1 n_2\), we obtain
\begin{equation*}\begin{split}
\left|\sumij s^{ij}_{ab}\right| &\leq \sqrt{\frac{8}{3} \log \frac{2}{\delta} \sigma^2} \leq \left(\frac{1}{2} - \rhoInfty\right) \|\PT \Dkminus\|_\infty,
\end{split}\end{equation*}
with probability at least \(1 - \lemInfSingleDelta\).
Therefore, by the union bound, with probability at least \(1 - \lemInfDelta\), we have
\begin{equation*}\begin{split}
\|(\PT \ROkC \PT - \PT \PC \PT)(\PT \Dkminus)\|_\infty \leq \left(\frac{1}{2} - \rhoInfty\right) \|\PT \Dkminus\|_\infty.
\end{split}\end{equation*}
\end{proof}
\subsection{Proof of Theorem~\ref{thm:exact-recovery-guarantee}}
\label{sec:org3d898a3}
\begin{proof}
The theorem immediately follows from the combination of
Lemma~\ref{lem:feasibility:main},
Lemma~\ref{lem:main-lemma-conc},
Lemma~\ref{lem:existence-dual-cert},
and the union bound.
\end{proof}
\section{Extension of Theorem~\ref{thm:cmc-feasible} to floor effects and varying thresholds \label{app:floor-effect}}
\label{sec:orgeef3ac8}
Theorem~\ref{thm:cmc-feasible} can be extended to the case where there are also \emph{floor effects} (clipping from below).
Here, we also allow the clipping thresholds to vary among entries.
Let \(\Lij\) denote the threshold for clipping from below and \(\Uij\) the threshold for clipping from above for entry \((i, j)\).
In this case, the trace-norm minimization algorithm is
\begin{equation*}\begin{split}
\widehat \M \in \mathop{\rm arg~min}\limits_{\X \in \matsp} \|\X\|_\trnrm \text{ s.t. } \begin{cases}
                                                                        \Mcij \leq X_{ij} &\text{ if } \Mcij = \Uij \text{ and } (i, j) \in \Omega, \\
                                                                        X_{ij} \leq \Mcij &\text{ if } \Mcij = \Lij \text{ and } (i, j) \in \Omega, \\
                                                                        X_{ij} = \Mcij &\text{ otherwise}, \\
                                                                        \end{cases}
\end{split}\end{equation*}
and the definition of \(\PC\) is
\begin{equation*}\begin{split}
(\PC(\Z))_{ij} = \begin{cases}
           Z_{ij} & \text{ if } \Lij < M_{ij} < \Uij, \\
           \max\{Z_{ij}, 0\} & \text{ if } M_{ij} = \Uij, \\
           \min\{Z_{ij}, 0\} & \text{ if } M_{ij} = \Lij, \\
           0 & \text{ otherwise},
           \end{cases}
\end{split}\end{equation*}
while \(\TrueNonClipped\) becomes
\begin{equation*}\begin{split}
\TrueNonClipped := \{(i, j): \Lij < \Mij < \Uij\}.
\end{split}\end{equation*}

The proof of Theorem~\ref{thm:cmc-feasible} in Appendix \ref{app:thm1} holds for this \(\PC\) without any further modification because
\begin{equation*}\begin{split}
|\PCij(z)| \leq |z|
\end{split}\end{equation*}
is satisfied for any \((i, j)\) and \(z \in \mathbb{R}\).
\section{Proof of Theorem~\ref{thm:dtr-guarantee} \label{app:thm2}}
\label{sec:org69d30e1}
Let \(\hM\) be the minimizer of Eq.~\eqref{eq:dtr-algorithm}.
Throughout the proof, the expectation operator \(\E\) is with respect to \(\Omega\), unless otherwise specified.

Define
\begin{equation*}\begin{split}
\LO(\X) := \sumO (\Cf(X_{ij})-M^\c_{ij})^2,\ \L(\X) := \E[\LO(\X)].
\end{split}\end{equation*}
Then we have
\begin{equation*}\begin{split}
\frac{1}{p} \L(\X) = \sumij(\Cf(X_{ij}) - M^\c_{ij})^2 = \|\Cf(\X) - \M^\c\|_\F^2.
\end{split}\end{equation*}
To obtain the theorem, we need to bound \(\frac{1}{p}\L(\hM)\) from above.
Our proof strategy is inspired by the analysis of \citet{Davenport1bit2014}.
\subsection{Basic Lemmas}
\label{sec:org14d4af5}
We will use the following lemma to prove Theorem~\ref{thm:dtr-guarantee}.
\begin{lemma}[]
Assume \(\M \in G\). Then for some (universal) constants \(C_0\) and \(C_1\),
\begin{equation*}\begin{split}
\Prob\left[\sup_{\X\in G}\left|\LO(\X)-\L(\X)\right| \ge \DTrLemMainTwo\right]\le \frac{C_1}{n_1 + n_2}
\end{split}\end{equation*}
holds.
\label{lem:dtr:main}
\end{lemma}
The proof of Lemma~\ref{lem:dtr:main} will be provided later.
\subsection{Proof of Theorem~\ref{thm:dtr-guarantee}}
\label{sec:orgfcdaac5}
We will show how we can derive Theorem~\ref{thm:dtr-guarantee} from Lemma~\ref{lem:dtr:main}.
\begin{proof}
First note that
\begin{equation*}\begin{split}
\L(\hM) &= \L(\hM) - \L(\M) \\
&= \L(\hM) - \LO(\hM) + \LO(\hM) - \L(\M) \\
&\leq \L(\hM) - \LO(\hM) + \LO(\M) - \L(\M) \\
&\leq 2 \sup_{\X \in G} |\L(\X) - \LO(\X)|.
\end{split}\end{equation*}
where we used \(\L(\M) = 0\) and the fact that \(\hM\) minimizes \(\LO\).

Therefore, by Lemma~\ref{lem:dtr:main}, with probability at least \(1-\frac{C_1}{n_1 + n_2}\), we have
\begin{equation}\label{eq:lem:dtr:main:proof:1}\begin{split}
\frac{1}{n_1 n_2}\|\Cf(\hM) - \M^\c\|_\F^2 = \frac{1}{p n_1 n_2} \L(\hM) \leq \DTrThmBound. \\
\end{split}\end{equation}

On the other hand, we have
\begin{equation}\label{eq:lem:dtr:main:proof:2}\begin{split}
\|\hM - \Cf(\hM)\|_\Fro \leq \|\hM - \Cf(\hM)\|_\trnrm \leq \|\hM\|_\trnrm + \|\Cf(\hM)\|_\trnrm \leq (\sqrt{\beta_1} + \sqrt{\beta_2})(k n_1 n_2)^{\frac{1}{4}}, \\
\end{split}\end{equation}
and
\begin{equation}\label{eq:lem:dtr:main:proof:3}\begin{split}
\|\M - \Mc\|_\Fro \leq \|\M - \Mc\|_\trnrm \leq \|\M\|_\trnrm + \|\Mc\|_\trnrm \leq (\sqrt{\beta_1} + \sqrt{\beta_2})(k n_1 n_2)^{\frac{1}{4}}.
\end{split}\end{equation}

Combining Eq.~\eqref{eq:lem:dtr:main:proof:1} with Eq.~\eqref{eq:lem:dtr:main:proof:2} and Eq.~\eqref{eq:lem:dtr:main:proof:3}, we obtain the theorem.
\end{proof}
\subsection{Proof of Lemma~\ref{lem:dtr:main}}
\label{sec:orga4cfe2e}
From here, we provide the proof of Lemma~\ref{lem:dtr:main}. It is based on the following two lemmas.

\begin{lemma}[{\citealp{Davenport1bit2014}}]
Let \(\mB \in \matsp\) be a matrix whose entries are independent Rademacher random variables, and \(\Delta \in \{0,1\}^{n_1\times n_2}\) be a random matrix for which \(\{\Delta_{ij}=1\}_{\fullij}\) occur independently with probability \(p\). Then there exists a universal constant \(C_1\), and for any \(h \ge 1\),
\begin{equation*}\begin{split}
\E\left[\left\|\mB\hadamard \Delta\right\|_\op^h\right] \le C_1 2^h (1+\sqrt{6})^h \left(p (n_1+n_2) + \log(n_1 + n_2)\right)^{h/2}.
\end{split}\end{equation*}
\label{lem:lem2}
\end{lemma}
The proof of Lemma~\ref{lem:lem2} is omitted, as it can be easily derived from \citep{Davenport1bit2014}.

We also want to bound the trace-norm of the Hadamard product of two matrices.
\begin{lemma}
Assume that there are two matrices \(\X\) and \(\Y\) of the same shape. Then we have
\begin{eqnarray*}
\|\X \hadamard \Y\|_\trnrm \le \unnCoherence(\X)^2 \|\X\|_\trnrm\|\Y\|_\trnrm.
\end{eqnarray*}
\label{lem:lem3}
\end{lemma}
The proof of Lemma~\ref{lem:lem3} will be provided later.

Based on the two lemmas above, we can show Lemma~\ref{lem:dtr:main}.
\begin{proof}[Proof of Lemma~\ref{lem:dtr:main}]
Using the Markov inequality, we have
\begin{eqnarray}\nonumber
&&\Prob\left[\supG\left|\LO(\X)-\L(\X)\right| \ge \DTrLemMainTwo\right]\\ \nonumber
&=&\Prob\left[\supG\left|\LO(\X)-\L(\X)]\right|^h \ge \left(\DTrLemMainTwo\right)^h\right] \label{eqn:markov} \\
&\le&\frac{\E\left[\supG\left|\LO(\X)-\L(\X)\right|^h\right]}{(\DTrLemMainTwo)^h}
\end{eqnarray}
thus we will bound the divided term, and then impute \(h\) with a specific value to get our result.

Defining \(\bomega\) by \(\oij := \Indicator\{(i, j) \in \Omega\}\), we can rewrite \(\LO(\X)\) as
\begin{equation*}\begin{split}
\LO(\X) = \sumij \oij (\Cf(X_{ij}) - \Cf(M_{ij}))^2.
\end{split}\end{equation*}

Therefore, by a symmetrization argument \citep[Lemma 6.3]{LedouxProbability1991}, we have
\begin{eqnarray*}
&\E\left[\sup_{\X \in G}\left|\LO(\X)-\L(\X)\right|^h\right] \le 2^h\E\left[\sup_{\X \in G}\left|\sumij B_{ij} \oij (\Cf(X_{ij})-\Cf(M_{ij}))^2\right|^h\right]
\end{eqnarray*}
where \(\{B_{ij}\}_\fullij\) are i.i.d.~Rademacher random variables, and the expectation in the upper bound is with respect to both \(\Omega\) and \(\{B_{ij}\}_{(i, j)}\). Then, we have
\begin{eqnarray*}
&&\E\left[\sup_{\X \in G}\left|\LO(\X)-\L(\X)\right|^h\right]\\
&\le&2^h\E\left[\sup_{\X \in G}\left|\sum_{ij}B_{ij}\oij (\Cf(X_{ij})-\Cf(M_{ij}))^2\right|^h\right]\\
&=&2^h \E\left[\sup_{\X \in G}\left|\langle \mB\hadamard \bomega, (\Cf(\X)-\Cf(\M))\hadamard(\Cf(\X)-\Cf(\M)) \rangle \right|^h\right] \\
&\le&2^h \E\left[\sup_{\X \in G}\|\mB\hadamard \bomega\|_\op^h \|(\Cf(\X)-\Cf(\M))\hadamard(\Cf(\X)-\Cf(\M))\|_\trnrm^h \right] \\
&\le&2^h\E\left[\|\mB\hadamard \bomega\|_\op^h\right] \sup_{\X \in G} \left\| (\Cf(\X)-\Cf(\M))\hadamard(\Cf(\X)-\Cf(\M))\right\|_\trnrm^h.
\end{eqnarray*}

According to Lemma~\ref{lem:lem2}, we have
\begin{eqnarray*}
\E\left[\|\mB\hadamard \bomega\|_\op^h\right] \le 2^h C_1 \left(1+\sqrt{6}\right)^h \left(p (n_1 + n_2)+ \log(n_1 + n_2)\right)^{h/2}.
\end{eqnarray*}
On the other hand, according to Lemma~\ref{lem:lem3},
\begin{eqnarray*}
&& \supG \left\|(\Cf(\X)-\Cf(\M)) \hadamard (\Cf(\X)-\Cf(\M))\right\|_\trnrm\\
&\le& \supG \|\Cf(\X)\hadamard \Cf(\X) +\Cf(\M)\hadamard\Cf(\M)-2\Cf(\M)\hadamard\Cf(\X)\|_\trnrm\\
&\le& \supG \left\{\|\Cf(\X)\hadamard \Cf(\X) \|_\trnrm+\|\Cf(\M)\hadamard\Cf(\M)\|_\trnrm+2\|\Cf(\M)\hadamard\Cf(\X)\|_\trnrm\right\} \\
&\le& \supG \left\{\unnCoherence(\Cf(\X))^2\|\Cf(\X)\|_\trnrm^2+\unnCoherence(\Cf(\M))^2\|\Cf(\M)\|_\trnrm^2+2\unnCoherence(\Cf(\M))^2\|\Cf(\M)\|_\trnrm\|\Cf(\X)\|_\trnrm \right\} \\
&\le& 4 \left(\supG \unnCoherence(\Cf(\X))^2\right) \left(\supG \|\Cf(\X)\|_\trnrm^2\right) \\
&\le& 4 \muG^2 \beta_2 \sqrt{k n_1 n_2}.
\end{eqnarray*}
Thus we have,
\begin{eqnarray*}
&&\E\left[\sup_{\X \in G}\left|\LO(\X)-\L(\X)\right|^h\right] \le \DTrLemMainOne.
\end{eqnarray*}

Plugging this into Eq.~\eqref{eqn:markov}, provided \(C_0 \ge 8(1+\sqrt{6})/e\), we obtain
\begin{eqnarray*}
&&\Prob\left[\sup_{\X \in G}\left|\LO(\X)-\L(\X)\right|\ge \DTrLemMainTwo\right]\\
&\le&\frac{\E\left[\sup_{\X \in G}\left|\LO(\X)-\L(\X)\right|^h\right]}{\left(\DTrLemMainTwo\right)^h} \\
&=& \frac{\DTrLemMainOne}{\left(\DTrLemMainTwo\right)^h} \\
&=& C_1 \left(\frac{8 (1 + \sqrt{6})}{C_0}\right)^h \\
&\le& C_1 e^h.
\end{eqnarray*}
by setting \(h=\log(n_1 + n_2)\), we get the lemma.
\end{proof}
\subsection{Proof of Lemma~\ref{lem:lem3}}
\label{sec:org69418a5}
\begin{proof}
If a rank-\(r\) matrix \(\X \in \matsp\) is expressed as \(\X=\mP\Q^\top\) with \(\mP \in \mathbb{R}^{n_1 \times r}\) and \(\Q \in \mathbb{R}^{n_2 \times r}\), then
\begin{eqnarray*}
\|\X \hadamard \Y\|_\trnrm \le \sum_{k=1}^r \|\mP_{\cdot, k}\| \|\Q_{\cdot, k}\| \sigma_k(\Y),
\end{eqnarray*}
holds \citep[Theorem 2]{HornNorm1995}, where \(\sigma_k(\cdot)\) denote the \(k\)-th largest singular value,
and \(\{\mP_{(i),\cdot}\}_{i=1}^{n_1}\) are descending rearrangement of \(\{\mP_{i, \cdot}\}_{i=1}^{n_1}\) with respect to its norm, i.e., \(\|\mP_{(1),\cdot}\| \geq \cdots \geq \|\mP_{(n_1),\cdot}\|\), and similarly for \(\{\Q_{(j),\cdot}\}_{j=1}^{n_2}\).
Let \(\X = \U \bSigma \V^\top\) be a skinny singular value decomposition of \(\X\), where \(\U \in \USp, \bSigma \in \SigmaSp\), and \(\V \in \VSp\). Then
\begin{eqnarray*}
\|(\U\bSigma)_{i,\cdot}\|^2=\sum_{l=1}^{r} U_{il}^2\sigma_l (\X)^2\le \sigma_1(\X)^2 \sum_{l=1}^r U_{il}^2 = \|\X\|_\trnrm^2 \|\U_{i,\cdot}\|^2 \le \|\X\|_\trnrm^2 \unnCoherence(\X)^2.
\end{eqnarray*}
Thus we have,
\begin{eqnarray*}
&&\|\X\hadamard \Y\|_\trnrm\le \sum_{k=1}^r \|(\U\bSigma)_{(k),\cdot}\| \cdot \|\V_{(k), \cdot}\| \sigma_k(\Y) \le \sum_{k=1}^r \|\X\|_\trnrm \unnCoherence(\X) \unnCoherence(\X) \sigma_k(\Y)=\unnCoherence(\X)^2\|\X\|_\trnrm \|\Y\|_\trnrm.
\end{eqnarray*}
\end{proof}
\end{appendices}
\end{document}